\documentclass[twocolumn]{custom_cls}          
\smartqed  
\usepackage{graphicx}
\usepackage{microtype}
\usepackage{subfigure}
\usepackage{booktabs} 
\usepackage{xcolor}
\usepackage{amsmath, amssymb}
\usepackage{multirow}

\usepackage{hyperref}
\usepackage{thm-restate}

\newcommand{\KL}{D_{\textnormal{KL}}}
\newcommand{\Tr}{\textnormal{Tr}}

\begin{document}

\title{Evaluation Metrics for Conditional Image Generation}


\author{Yaniv Benny \and Tomer Galanti \and Sagie Benaim \and Lior Wolf
}


\institute{Y. Benny \at
              Tel-Aviv University, Israel \\
              \email{yanivbenny@mail.tau.ac.il}           
           \and
           T. Galanti \at
              Tel-Aviv University, Israel \\
              \email{tomergalanti@mail.tau.ac.il}           
           \and
           S. Benaim \at
              Tel-Aviv University, Israel \\
              \email{sagiebenaim@mail.tau.ac.il}           
           \and
           L. Wolf \at
              Tel-Aviv University, Israel; Facebook AI Research \\
              \email{wolf@cs.tau.ac.il}           
}

\date{}

\maketitle

\begin{abstract}
We present two new metrics for evaluating generative models in the class-conditional image generation setting. These metrics are obtained by generalizing the two most popular unconditional metrics: the Inception Score (IS) and the Fr\'{e}chet Inception Distance (FID). A theoretical analysis shows the motivation behind each proposed metric and links the novel metrics to their unconditional counterparts. The link takes the form of a product in the case of IS or an upper bound in the FID case. We provide an extensive empirical evaluation, comparing the metrics to their unconditional variants and to other metrics, and utilize them to analyze existing generative models, thus providing additional insights about their performance, from unlearned classes to mode collapse.
\end{abstract}

\newpage
\section{Introduction} \label{introduction}
Unconditional image generation models have seen rapid improvement both in terms of generation quality and diversity.
These generative models are successful, if the generated images are indistinguishable from real images sampled from the training distribution. This property can be evaluated in many different ways, the most popular are the Inception Score (IS)~\cite{salimans2016improved}, which considers the output of a pretrained classifier, and the  Fr\'{e}chet Inception Distance (FID)~\cite{heusel2017gans}, in which measures the distance between the distributions of extracted features of the real and the generated data.

While unconditional generative models take as input a random vector, conditional generation allows one to control the class or other properties of the synthesized image. In this work, we consider class-conditioned models, introduced in \cite{mirza2014conditional}, where the user specifies the desired class of the generated image.
Employing unconditional metrics, such as IS and FID, in order to evaluate conditional image generation fails to take into account whether the generated images satisfy the required condition. 
On the other hand, classification metrics, such as accuracy and precision, currently used to evaluate conditional generation, have no regard for image quality and diversity.

One may opt to combine the unconditional generation and the classification metrics to produce a valuable measurement for conditional generation. However, this suffers from a few problems. First, the two components are of different scales and the trade-off between them is unclear. Second, they do not capture changes in variance within the distribution of each class. {\color{black}To illustrate this, consider Fig.~\ref{fig:fid_classifier}, which depicts two different distributions for polar coordinates $(R_i + \alpha, 0.1\beta )$ where $R_i$ is a different radius for each class $i \in \{1,2\}$, $\alpha \sim \mathcal{N}(0, 0.1^2)$, $\beta \sim U(0,2\pi)$. Distribution `A' has $R_1=1, R_2=3$, therefore it has a zero mean, for each of the classes, and a standard deviation 1.0 for the first class and 3.0 for the second class independently in each axis, which results in a standard deviation of 2.0 for the entire distribution. Distribution `B' corresponds to radii $R_1=1.5, R_2=2.5$, and therefore it also has a zero mean but a standard deviation of $1.5$ for the first class and $2.5$ for the second class in each axis, which again, results in a standard deviation of 2.0 for the entire distribution. Since the FID score compares distributions by their mean vectors and covariance matrices, the FID (between distributions `A' and `B') is zero. The classification error (the optimal classifier is shown in green) is also zero, despite the in-class distributions being different.}

In order to provide valuable metrics to evaluate and compare conditional models, we present two metrics, called Conditional Inception Score (CIS) and Conditional Fr\'{e}chet Inception Distance (CFID). The metrics contain two components each: (i) the within-class component (WCIS/WCFID) measures the quality and diversity for each of the conditional classes in the generated data. In other words, it measures the ability to replicate the distribution of each class in the true samples; (ii) the between-class component (BCIS/BCFID) measures how close the representation of classes in the generated distribution is to the representation in the real data distribution. See Fig.~\ref{fig:separation} for an illustration.

In contrast to the combined FID and classifier, our WCFID and BCFID components of the FID are both larger than zero for the example in Fig.~\ref{fig:fid_classifier}, successfully capturing the differences between the distributions. 

Our analysis shows direct links between the novel conditional metrics and their unconditional counterparts. The (unconditional) Inception Score can be decomposed to a multiplication between BCIS and WCIS. We further show that due to the bounded region of the metrics, this translates to a trade-off between BCIS and WCIS and that each one of them form a tight lower bound on the IS. In the analysis of the FID score, we show that the sum of WCFID and BCFID forms a tight upper bound of the FID.

After analyzing the metrics, we performed various experiments to ground the theoretical claims and to highlight the role of the new metrics in evaluating conditional generation models. 
First, a set of simulations was conducted, in which we performed label noising, image noising, image manipulation, and simulated mode collapse. Under all conditions, our methods came out as the most sensitive to the applied augmentations. We then evaluated several pretrained models of popular architectures on various datasets and training schemes using the proposed scores and identified significant insights that were detected by our metrics. Our metrics were found to be a decisive factor to determine the generation performance in each dataset.

\begin{figure}[t]
\centering
\includegraphics[width=\linewidth, trim={80 30 60 17}, clip]{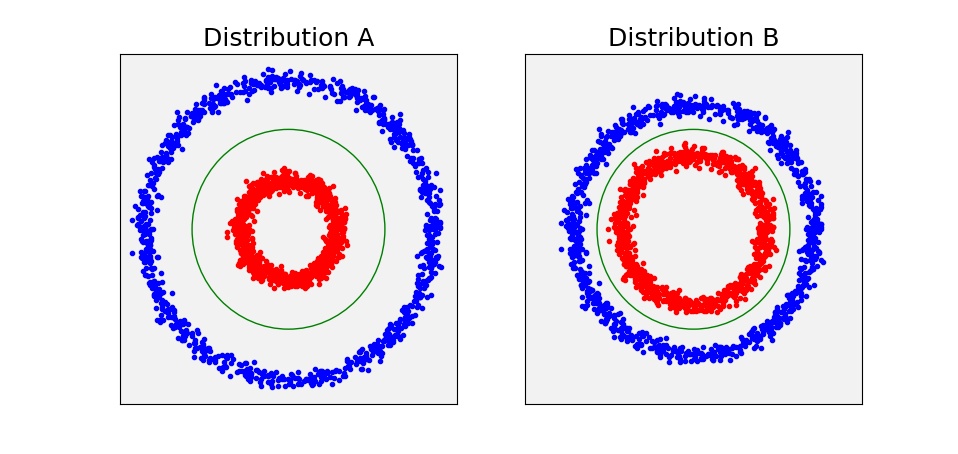}
\caption{A case against measuring success by relying on FID combined with a classification score. Shown are two distributions, each with two classes, in a two-dimensional feature space. The green circle acts as the classifier. For both distributions, the overall mean and variance are equal, therefore, the FID is zero. The classification error is also zero, and the two distributions are, therefore, indistinguishable by this score as well. However, within the classes we see a shift in variance in the second distribution compared to the first. Our proposed WCFID and BCFID metrics both show values above zero and, therefore, detect the difference between the distributions.}
\label{fig:fid_classifier}
\end{figure}

\begin{figure}[t]
\centering
\begin{tabular}{cc}
\includegraphics[width=0.46\linewidth, trim={35 80 30 80}, clip]{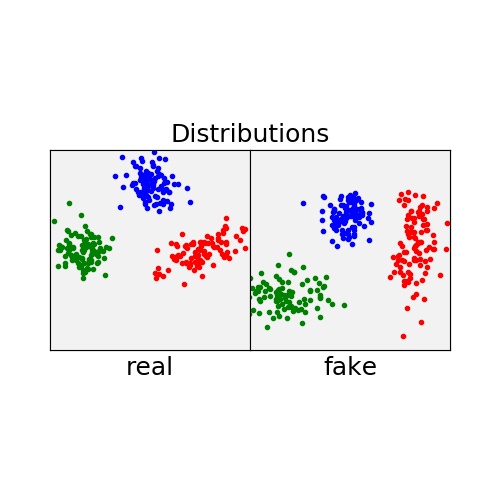} &
\includegraphics[width=0.46\linewidth, trim={35 80 30 80}, clip]{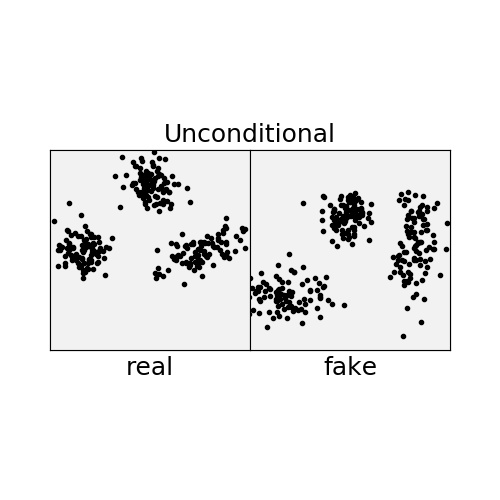} \\
(a)&(b)\\
\includegraphics[width=0.46\linewidth, trim={35 80 30 80}, clip]{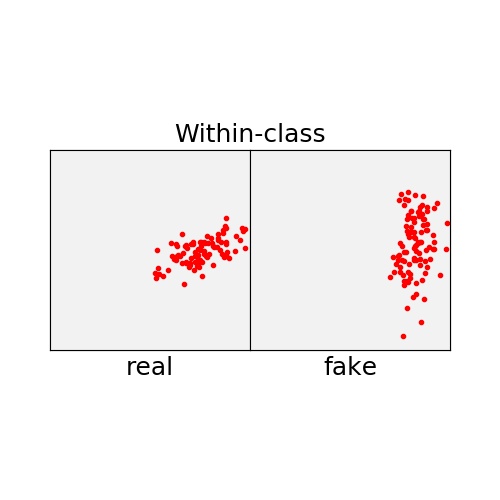} &
\includegraphics[width=0.46\linewidth, trim={35 80 30 80}, clip]{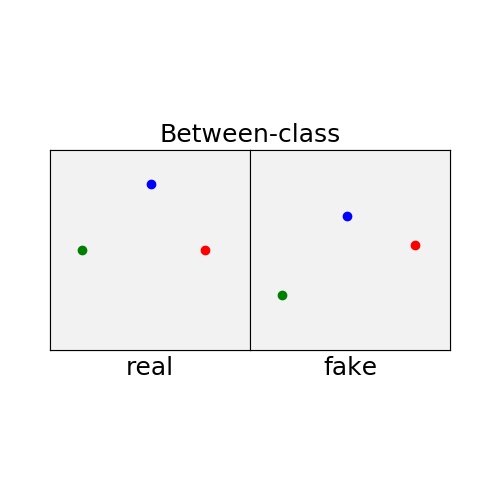} \\
(c)&(d)\\
\end{tabular}
\caption{The differences between unconditional, within-class and between-class evaluations for two given distributions with labelled samples. (a) Sample distributions. (b) The unconditional evaluation disregards the labels and compares the distance between the distributions. (c) The within-class evaluation compares each class in the first distribution with the corresponding class in the second (shown for one class). (d) The between-class evaluation compares the distribution of class averages.}
\label{fig:separation}
\end{figure}

\subsection{Related Work}

\noindent{\bf Generative Models}
Generative models, and in particular Generative Adversarial Networks~\cite{GAN} aim to generate realistic looking images from a target distribution, while capturing the diversity of images. Advances in the loss and architecture allowed for improved quality and diversity of generation. For instance, \cite{pmlr-v70-arjovsky17a,10.5555/3295222.3295327} attempt to minimize the Wasserstein distance between the generated and real distributions. This allowed for improved variability in generation, in particular by reducing mode collapse.
On the architectural side, Progressive GANs~\cite{karras2018progressive}, StyleGAN~\cite{karras2019style} and StyleGANv2~\cite{Karras2019stylegan2}, introduced advanced architectures and training methods, allowing for further improvements. 

\noindent {\bf Conditional Generation}
In conditional generation, control over the generation is provided, e.g., by class-conditioning~\cite{mirza2014conditional,chen2016infogan,singh2019finegan}, a given text ~\cite{zhang2017stackgan,xu2018attngan}, requiring specific semantic features~\cite{johnson2018image}, or finding analogs to images from another distribution \cite{isola2017image,zhu2017unpaired}.
The recent state of the art in class-conditional generation, which is the BigGAN method~\cite{brock2018large}, can learn conditional representation on ImageNet~\cite{russakovsky2015imagenet} with high quality and diversity.

To train these models, several changes have been proposed to the unconditional method. CGAN~\cite{mirza2014conditional} injects the conditional component to the discriminator along with the image, ACGAN~\cite{odena2017conditional} added an auxiliary classifier tasked to accurately predict the conditioned label, SGAN~\cite{odena2016semi} modified the discriminator output to detect real classes while treating fake images as an additional class. A special unsupervised setting was proposed by InfoGAN~\cite{chen2016infogan}, where the condition is unlabelled in the real data and the model constructs a disentangled representation by maximizing the mutual information between the conditioned variable and the observation.

\noindent {\bf Evaluation Metrics}
To evaluate different models in terms of high quality generation and diversity, evaluation metrics were proposed for the unconditional generation setting. 
The Inception Score (IS)~\cite{salimans2016improved} uses the predictions of a pretrained classifier, InceptionV3~\cite{Szegedy2015RethinkingTI}, to assess: 1. Quality: whether the conditional probability of a generated sample $G(z)$ over the labels $y$, is highly predictable (low entropy) and 2. Diversity: whether the marginal probability of labels over all generated samples is highly diverse (high entropy). 
The Fr\'{e}chet Inception Distance (FID)~\cite{heusel2017gans}, was proposed as an alternative to the IS, by considering the distribution of features of real data and generated data. FID models these distributions as multivariate Gaussian distributions and measures the distance between them. FID was shown to be sensitive to mode collapse and more robust to noise than IS. Additional metrics, such as Perceptual Path Length (PPL)~\cite{karras2019style} and Kernel Inception Distance (KID)~\cite{binkowski2018demystifying} were also introduced. 
Still, the IS and FID are the most widely accepted metrics for image generation. 

Nevertheless, these measures are designed for the unconditional setting, and for class-conditional, they do not assess the level at which the categorical condition manifests itself in the generated data. In this work, we extend the IS and FID to the class-conditional setting, showing their relation to their unconditional counterparts and demonstrate the usefulness of these metrics in the conditional setting.

{\color{black}Accuracy-based evaluation methods have been proposed to evaluate the conditional generation capabilities of a model. For example, by measuring the accuracy of predicting the conditioned label with a pretrained classifier. However, accuracy is not a reliable benchmark for image generation, since only a slight higher probability for the correct class is necessary for correct prediction, and the image can be far from looking realistic or truly depicting the required condition. Another approach, called Classification Accuracy Score (CAS)~\cite{ravuri2019classification} has been proposed. Instead of measuring the accuracy of the generated images on a pretrained classifier, the generated images are used to train a classifier that is then used to predict the labels of the real images. This reversed approach has been shown to be more promising, since the images need to contain the necessary features to define the represented class, in order for the classifier to work on real images as well. However, it requires training a different classifier for each generative model, which can result in a biased evaluation, since different models can benefit from different classifier hyperparameters. The need to retrain classifiers for each evaluation is laborious and can prevent the method from being used, e.g., for selecting the best epoch during training.} 

{\color{black}Some recent attempts have been made to asses conditional generation with modified versions of the IS and FID. A Conditional Inception Score for the Image-to-Image translation task was proposed by Huang et al.~\cite{huang2018multimodal}. This method measures the difference in the probability distribution of images before and after a translation to another category. While this method relies on the IS equation as a tool for measurement, it does not follow the principles behind the IS, i.e. low entropy for a single image and high entropy on average. Instead, it only measures how much the prediction of each image has changed. In addition, this method is only applicable to image-to-image translation and not for conditional image generation. Miyato et al.~\cite{miyato2018cgans} were the first to our knowledge to measure the intra-FID, which is equal to our WCFID component. Their presentation of this metric lacks thorough experiments and is presented without any theoretical analysis. Nevertheless, these previous attempts to establish conditional metrics confirm the need for conditional evaluations scores and point to the possible ingredients of such solutions. Our work contributes to both approaches by providing justification and explanation, and also by building unified metrics.}

\section{Problem Setup}
{\color{black}The following setup addresses Generative Adversarial Networks, since they are currently, along with Variational Autoencoders, the most popular methods for image generation. However, the evaluation metrics are not limited to any model and can be applied to any class-conditionally generated images. The only requirement is that each image (either real or fake), corresponds to one class from a finite set of classes.
}

We consider a distribution of real samples:
\begin{equation}\label{eq:cond}
c \sim \mathcal{D}_C, \quad x \sim \mathcal{D}^c_R,
\end{equation}
where $\mathcal{D}_C$ is a distribution over a set of $K>1$ classes and $\mathcal{D}^c_{R}$ is the conditional distribution of a sample $x\in \mathbb{R}^n$ taken from the class $c$. The algorithm is provided with a dataset of i.i.d labelled examples $\mathcal{S} = \{(x_i,c_i)\}^{m}_{i=1}$ that were sampled from the generative process in Eq.~\ref{eq:cond}. In addition, the distribution $\mathcal{D}_C$ of classes and its corresponding probability density function $p(c)$ are known or assumed to be uniform. The distribution of real samples $x$ marginalized over $c \sim \mathcal{D}_C$ is denoted by $\mathcal{D}_R$ and the corresponding probability density function by $p_R(x)$.

In {\bf conditional generation}, the algorithm learns a generative model $G$ that tries to generate samples that are similar to the real samples in $\mathcal{D}_R$. The generative model takes a random seed $z \sim \mathcal{D}_Z \subset \mathbb{R}^{d}$ and a class $c \sim \mathcal{D}_C$ as inputs and returns a generated sample $G(z,c)$. Here, $\mathcal{D}_Z$ is a pre-defined distribution over a latent space $\mathbb{R}^d$ of dimension $d < n$, where $n$ is the dimensionality of the samples. Typically $\mathcal{D}_Z$ is the standard normal distribution. We denote by $\mathcal{D}_G$ the distribution of generated samples.

In conditional generation, we are interested in two aspects of the generation. Images from the same conditioned variable should be of the same class in $\mathcal{D}_R$, and different latent variables $z$ should cover the range of each class.

The {\bf category discovery} setting is a special case of the conditional generation, proposed in~\cite{chen2016infogan}, where the algorithm is provided with a set of unlabelled samples $\mathcal{S} = \{x_i\}^{m}_{i=1}$. The algorithm is still aware of the existence of the partition of the data into classes that are distributed according to $\mathcal{D}_C$. The goal of the algorithm is to generate samples that are similar to the real samples and also have them clustered in a proper manner into $K$ clusters.

\subsection{Inception Score} 
The Inception Score (IS) is a method for measuring the realism of a generative model's outputs. For a given generative model $G$, a latent vector $z \sim \mathcal{D}_Z$ and a random class $c\sim \mathcal{D}_C$, we apply a pretrained classifier on the generated image $x = G(z,c)$ to obtain a distribution over the labels, which is denoted by $p_G(y|x)$. We denote the corresponding random variables by $Z$, $C$, $X = G(Z,C)$ and $Y \sim p_G(y|X)$, the distribution of $X$ by $\mathcal{D}_G$ and the probability density function of $x$ by $p_G(x)$. Images that contain meaningful objects should have a condition distribution $p_G(y|x)$ of low entropy. Furthermore, we expect the model to generate varied images, so the marginal distribution $p_G(y) := \mathbb{E}_{z,c} [p_G(y|x=G(z,c))]$ should have a high entropy. The Inception Score is computed as:
\begin{equation}
IS(X;Y) := \exp\{\mathbb{E}_{x\sim \mathcal{D}_G}[\KL(p_G(y|x) \| p_G(y)]\},
\end{equation}
where $\KL(p\|q)$ is the KL-divergence between two probability density functions. A high score indicates both a high variety in data and that the images are meaningful.

The Inception Score can also be formulated using the mutual information between the generated samples and the class labels:
\begin{equation}\label{eq:mutualIS}
IS(X;Y) = \exp\{I(X;Y)\},
\end{equation}
where $I(X;Y)$ is the mutual information between $X$ and $Y$. As can be seen, by maximizing the IS, one maximizes the mutual information between $X$ and $Y$. However, this equation indicates that the IS is not sufficient in order to evaluate generative models in the conditional generation settings, since the score does not take the conditioned class into account.

Due to the properties of the mutual information, it can be seen that for a domain with $K$ classes, the score is within the range $[1,K]$.

\subsection{Fr\'{e}chet Inception Distance} 
The Fr\'{e}chet distance $d^2(\mathcal{D}_1,\mathcal{D}_2)$ between two distributions $\mathcal{D}_1,\mathcal{D}_2$ is defined by:
\begin{equation}
d^2(\mathcal{D}_1,\mathcal{D}_2) := \min_{X,Y}\mathbb{E}_{X,Y}[\|X-Y\|^2],
\end{equation}
where the minimization is taken over all random variables $X$ and $Y$ having marginal distributions $\mathcal{D}_1$ and $\mathcal{D}_2$, respectively. In general, the Fr\'{e}chet distance is intractable, due to its minimization over the set of arbitrary random variables. Fortunately, as shown by~\cite{RePEc:eee:jmvana:v:12:y:1982:i:3:p:450-455}, for the special case of multivariate normal distributions $\mathcal{D}_1$ and $\mathcal{D}_2$, the distance takes the form:
\begin{equation}
d^2(\mathcal{D}_1,\mathcal{D}_2) = \|\mu_1 - \mu_2 \|^2 + \Tr(\Sigma_1 + \Sigma_2 - 2(\Sigma_1 \Sigma_2)^\frac{1}{2})
\end{equation}
where $\mu_i$ and $\Sigma_i$ are the mean and covariance matrix of $\mathcal{D}_i$. The first term measures the distance between the centers of the two distributions. The second term: 
\begin{equation}
d_0(\mathcal{D}_1,\mathcal{D}_2) := \Tr(\Sigma_1 + \Sigma_2 - 2(\Sigma_1 \Sigma_2)^\frac{1}{2}),
\end{equation} 
defines a metric on the space of all covariance matrices of order $n$. 

For two given distributions $\mathcal{D}_R$ of real samples and $\mathcal{D}_G$ of the generated data, the FID score~\cite{heusel2017gans} computes the Fr\'{e}chet distance between the real data distribution and generated data distribution using a given feature extractor $f$ under the assumption that the extracted features are of multivariate normal distribution:
\begin{equation}
\begin{aligned}
    FID(\mathcal{D}_R,\mathcal{D}_G) :=& d^2(f \circ \mathcal{D}_R, f\circ \mathcal{D}_G) \\
    =& \|\mu^R - \mu^G\|^2 \\
    &+ \Tr(\Sigma^R + \Sigma^G - 2(\Sigma^R\Sigma^G)^\frac{1}{2}),
\end{aligned}
\end{equation}
where $\mu^R, \Sigma^R$ and $\mu^G, \Sigma^G$ are the centers and covariance matrices of the distributions $f \circ \mathcal{D}_R$ and $ f\circ \mathcal{D}_G$, respectively. For evaluation, the mean vectors and covariance matrices are approximated through sampling from the distribution.

\section{Method}

In this section, we introduce the class-conditioned extensions of the Inception Score and FID.

\subsection{Conditional Inception Score}\label{CIS}

The conditional analysis of the Inception Score addresses both aspects of conditional generation: the need to create realistic and diverse images, and the need to have each generated image match its condition. We define two scores: the between-class (BCIS) and the within-class (WCIS).

BCIS evaluates the IS on the class averages. It is a measurement of the mutual information between the conditioned classes and the real classes. The prediction probabilities for all the samples in each conditioned class are averaged to produce the average prediction probability of the entire class, then the IS is computed on these averages.

The BCIS is defined in the following manner:
\begin{equation}
\begin{aligned}
BCIS(X;Y) :=& IS(C;Y) \\
=& \exp\left\{\mathbb{E}_c[\KL(p_G(y|c) \| p_G(y))] \right\} \\
\end{aligned}
\end{equation}
where,
\begin{equation}
\begin{aligned}
p_G(y|c) &= \frac{1}{p(c)} \cdot \mathbb{E}_{x \sim \mathcal{D}_G} [p_G(y,c|x)] \\
&= \mathbb{E}_{x \sim \mathcal{D}^c_G} [p_G(y|x)] \\
\end{aligned}
\end{equation}

WCIS evaluates the IS within each category. It is a measurement of the mutual information between the real classes conditioned on the samples and the real classes conditioned on the conditioned classes. The final score is the geometric average score over all the classes, which is equivalent to the exponent on the arithmetic average of the mutual information over all the classes. To define this measure, we define two random variables $X_c := (X|C=c)$ and $Y_c := (Y|C=c)$ which are the random variables $X$ and $Y$ conditioned on the class being $c$. 

The WCIS is defined as:
\begin{equation}
WCIS(X;Y) := \exp\{\mathbb{E}_{c}[I(X_c;Y_c)] \},
\end{equation}
where the mutual information is computed as follows:
\begin{equation}
I(X_c;Y_c) = \mathbb{E}_{x \sim \mathcal{D}^c_G}[\KL(p_G(y|x) \| p_G(y|c))],
\end{equation}
where $\mathcal{D}^c_G$ is the distribution of $X_c$. 

In general, we wish the BCIS to be as high as possible and the WCIS to be as low as possible. High BCIS indicates a distinct class representation for each conditioned class and a wide coverage across the conditioned classes, which is a desired property. High WCIS indicates a wide coverage of real classes within the conditioned classes, which is an undesired property, since each conditioned class should represent only a single real class. In this way, one obtains consistent prediction within each class and has high variability between classes.

The following theorem presents the compositional relationship between IS and the proposed conditional measures.
\begin{theorem} 
\label{thm:is}
Let $C\sim \mathcal{D}_C$ and $Z\sim \mathcal{D}_Z$ be two independent random variable. Let $X = G(Z,C)$ for a continuous generator function $G$ and let $Y$ be a discrete random variable distributed by $p(y|X)$. Then,
\begin{equation}\label{eq:ISeqCIS}
IS(X;Y) = BCIS(X;Y) \cdot WCIS(X;Y)
\end{equation}
The proof is provided in the appendix.
\end{theorem}

By definition, as with the IS, both BCIS and WCIS lie within $[1,K]$. Since we wish IS to be as large as possible and both BCIS and WCIS lie in the same interval, the theorem asserts that there is a tension between the BCIS and WCIS measures, since both of them cannot be large at the same time. In addition, since both components are larger than $1$, the theorem shows that they both provide a lower bound on the IS and the bound is tight when the other component is equal to $1$. The final realization is that the IS can be very high even when the BCIS component is low, simply by having a high WCIS. This gap between IS and BCIS indicates bad conditional representation which is overlooked by the unconditional evaluation.

On these grounds, we propose the BCIS and WCIS together as the conditional alternative to the IS. Each metric shows a different property of the generated data and, as shown in the theorem, the IS is readily obtained by multiplying the conditional components.

\subsection{Conditional Fr\'{e}chet Inception Distance}\label{CFID}
For conditional FID, we want to measure the distance between different distributions, according to the feature vector $f(x)$, produced by the pre-trained feature extractor $f$ on a sample $x$. Analogous to the conditional IS metrics, we measure the between-class distance between averages of conditioned class features and averages of real class features, as well as the average within-class distance for each matching pair of real and conditioned classes.

BCFID measures the FID between the distribution of the average feature vector of conditioned classes in the generated data and the distribution of the average feature vector of real classes in the class real data. It evaluates the coverage of the conditioned classes over the real classes.

For each distribution specifier $E \in \{R,G\}$, we estimate the per-class mean $\mu^{E}_c$, the mean of means $\mu^{E}_B$, and the covariance of the feature vectors $\Sigma^E_B$.
\begin{align}
&\mu^{E}_c = \mathbb{E}_{x \sim \mathcal{D}^c_E}[f(x)] \\
&\mu^{E}_B = \mathbb{E}_{c \sim \mathcal{D}_C}[\mu^{E}_c] = \mathbb{E}_{x \sim \mathcal{D}_E}[f(x)] = \mu^{E} \\
&\Sigma^E_B = \mathbb{E}_{c\sim \mathcal{D}_C}[(\mu^E_c - \mu_B^E)(\mu^E_c - \mu_B^E)^{\top}]
\end{align}

The BCFID is defined as:
\begin{equation}
\begin{aligned}
BCFID(\mathcal{D}_R,\mathcal{D}_G) :=& \|\mu_B^R - \mu_B^G\|^2 \\
+& \Tr(\Sigma^R_B + \Sigma^G_B - 2(\Sigma^R_B\Sigma^G_B)^\frac{1}{2})
\end{aligned}
\end{equation}

WCFID measures the FID between the distribution of the generated data and the real data within each one of the classes. It evaluates how similar each conditioned class is to its respective real class. The total score is the mean FID within the classes.

For each distribution specifier $E \in \{R,G\}$, the within-class covariance matrices are defined as:
\begin{equation}
\begin{aligned}
\Sigma^E_{W,c} = \mathbb{E}_{x \sim \mathcal{D}^c_E}[(f(x) - \mu^E_{c})(f(x) - \mu^E_{c})^T] \label{eq:covariancewc}
\end{aligned}
\end{equation}

The WCFID is defined as:
\begin{equation}
\begin{aligned}
    &WCFID(\mathcal{D}_R,\mathcal{D}_G) \\
    :=& \mathbb{E}_{c \sim \mathcal{D}_C} [FID(\mathcal{D}^c_{R},\mathcal{D}^c_{G})]  \\
    :=& \mathbb{E}_{c \sim \mathcal{D}_C} \Big[\|\mu^R_{c} - \mu^G_{c}\|^2 + \\
     & \quad\quad\quad \Tr(\Sigma^R_{W,c} + \Sigma^G_{W,c} - 2(\Sigma^R_{W,c}\Sigma^G_{W,c})^\frac{1}{2})\Big] \\
\end{aligned}
\end{equation}

Note that we compare between matching pairs of conditioned and real classes. When a mapping between conditioned and real classes exists, i.e., in conditional GANs, this is straightforward. In the case when there is no such mapping, i.e., in the class discovery case, such as when employing the InfoGAN method, a mapping needs to be created. For example, this can be done by using a classifier to get the prediction probabilities for the generated images. Then average the probabilities for each conditioned class and apply the Hungarian algorithm on the average probabilities.

In general, the desire is to minimize both component, since each computes a different aspect of the distance between the real and the generated distributions. {\color{black}BCFID measures the distance between the real and the fake class averages, therefore it measures the coverage of the classes. WCFID on the other hand, measures the distance between the real and fake samples within each class. Therefore, it measures the similarity and diversity. This means that the two components complement each other.}

The following theorem ties the FID and the conditional FID components.
\begin{theorem}\label{thm:fid} 
Let $\mathcal{D}_R$ and $\mathcal{D}_G$ be the distributions of real and generated samples. Then,
\begin{equation}\label{eq:FIDeqCFID}
FID(\mathcal{D}_R,\mathcal{D}_G) \leq BCFID(\mathcal{D}_R,\mathcal{D}_G) + WCFID(\mathcal{D}_R,\mathcal{D}_G)
\end{equation}
and the bound is tight under certain conditions. 
The proof is provided in the appendix.
\end{theorem}

By this theorem, in conditional generation, FID gives an optimistic evaluation to the model that ignores bad cases. A good unconditional score can be obtained even though there is a considerable friction between the real and generated distributions in terms of conditional generation. This friction can occur either by bad representation of classes (high BCFID) or unmatching diversity within classes (high WCFID). For this reason, we propose the BCFID and WCFID as the conditional alternative to the FID. In addition to providing two meaningful scores that are similarly scaled, an upper bound to the FID can be computed by adding the two components. {\color{black}Since a good conditional generation corresponds to the case where both BCFID and WCFID tend to be small, we suggest using $BCFID(\mathcal{D}_R,\mathcal{D}_G) + WCFID(\mathcal{D}_R,\mathcal{D}_G)$ as a single-valued measure of conditional generation. }

{\color{black}
\subsection{Within-class Model Analysis}
We note that the WCIS and WCFID, proposed in this work, reduce the evaluations on the various classes into single values. While this is beneficial to summarize the performance of a model as much as possible, this misses an opportunity to inspect the performance of the model on each class separately. By looking at the IS/FID component of each class before the averaging into the WCIS/WCFID, we can reveal which classes contribute to the performance of the model and which classes are not generated well. This can be a valuable insight during training and fine-tuning of models. In Sec.~\ref{biggan}, we perform such in-depth analysis with the BigGAN~\cite{brock2018large} on the ImageNet~\cite{russakovsky2015imagenet} dataset.
}

\begin{figure*}[t]
\centering
\begin{tabular}{cc}
\multicolumn{2}{c}{
\includegraphics[width=0.90\linewidth, trim={170 280 140 250}, clip]{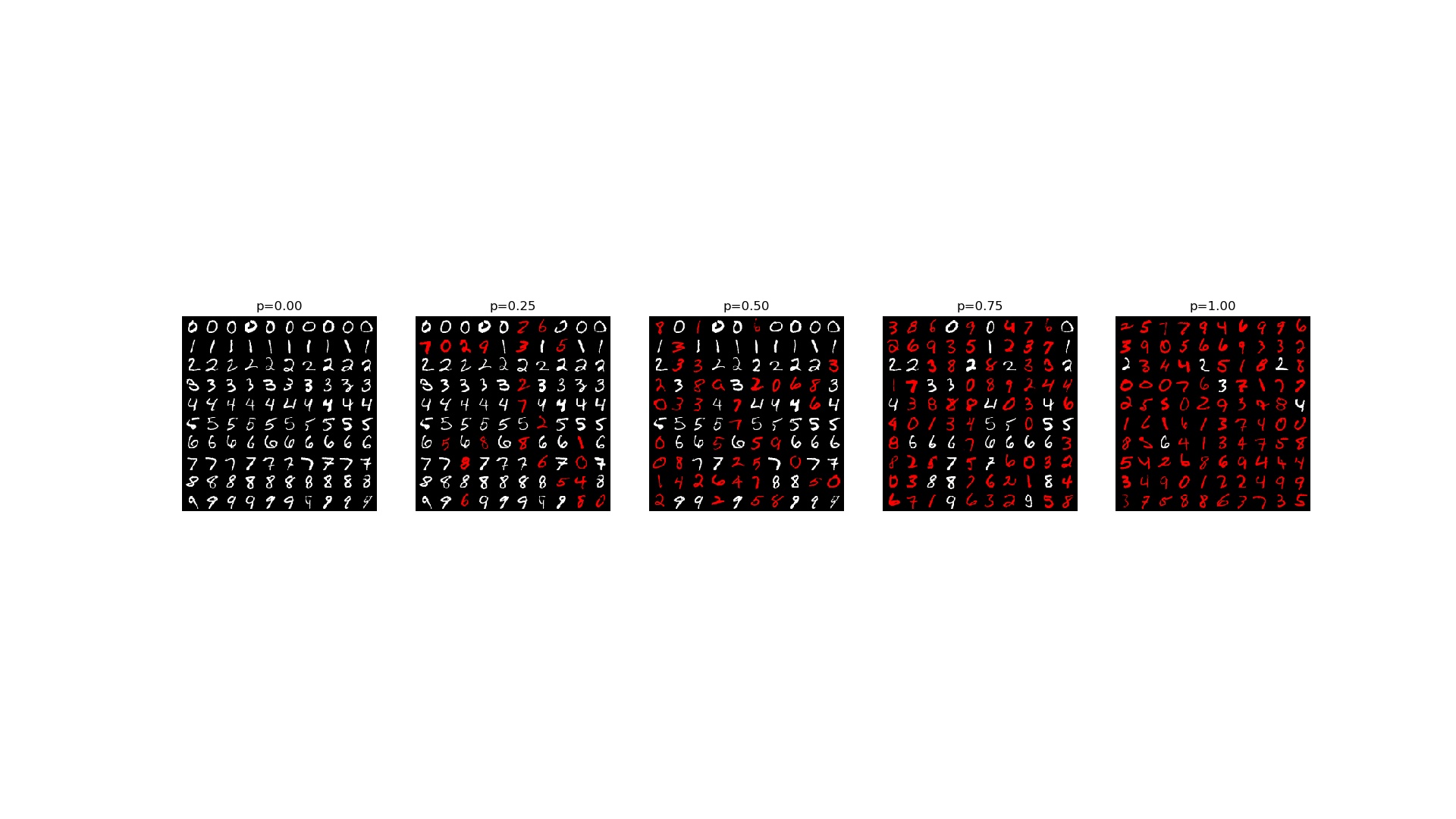}} \\
\multicolumn{2}{c}{(a)} \\
\includegraphics[width=.45\linewidth, trim={10 0 40 10}, clip]{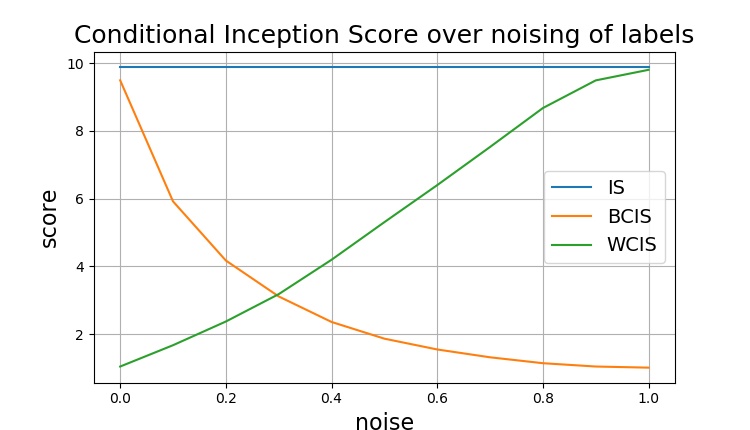} &
\includegraphics[width=.45\linewidth, trim={10 0 40 10}, clip]{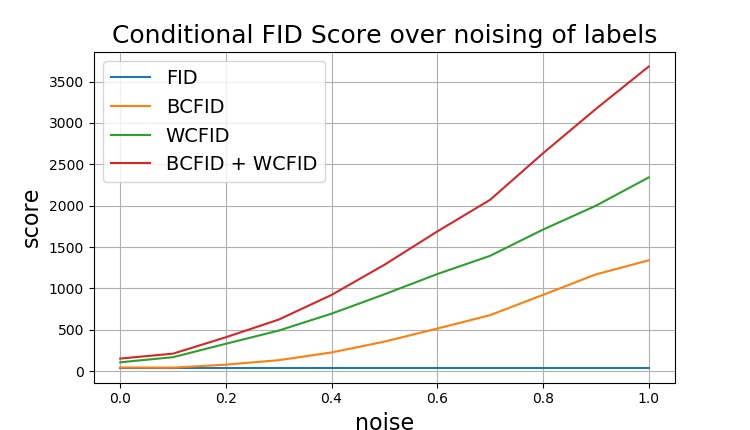} \\
(b)&(c)\\
\end{tabular}
\caption{{\bf Label Noising:} Labels were randomly replaced with probability $p$ to simulate bad conditional generation. (a) Each row has images conditioned on the same class. Images in red indicate bad conditional generation. (b) The effect of label noising on the unconditional and conditional IS metrics as a function of noise. (c) Same for the conditional FID score.}
\label{fig:label_noising}
\end{figure*}

\begin{figure*}[t]
\centering
\begin{tabular}{ccc}
Gaussian & Salt \& Pepper & Permutation\\
\includegraphics[width=0.3045\linewidth, trim={0 0 2 0}, clip]{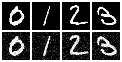} &
\includegraphics[width=0.3045\linewidth, trim={0 0 2 0}, clip]{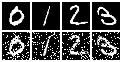} &
\includegraphics[width=0.3045\linewidth, trim={0 0 2 0}, clip]{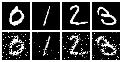} \\
\includegraphics[width=0.3045\linewidth, trim={10 0 52 37}, clip]{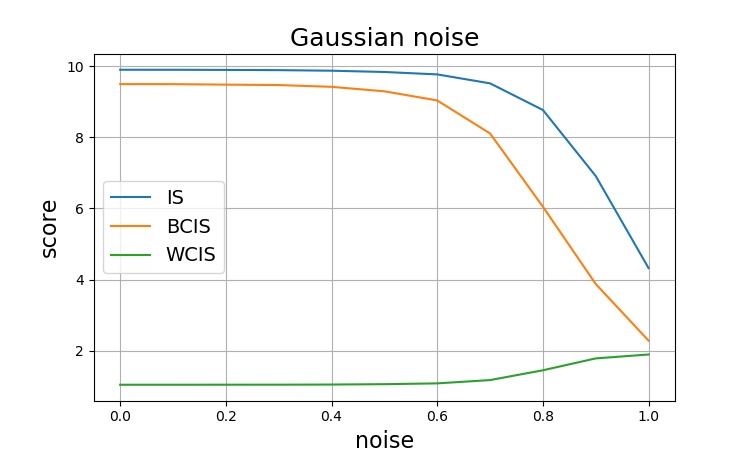} &
\includegraphics[width=0.3045\linewidth, trim={10 0 52 37}, clip]{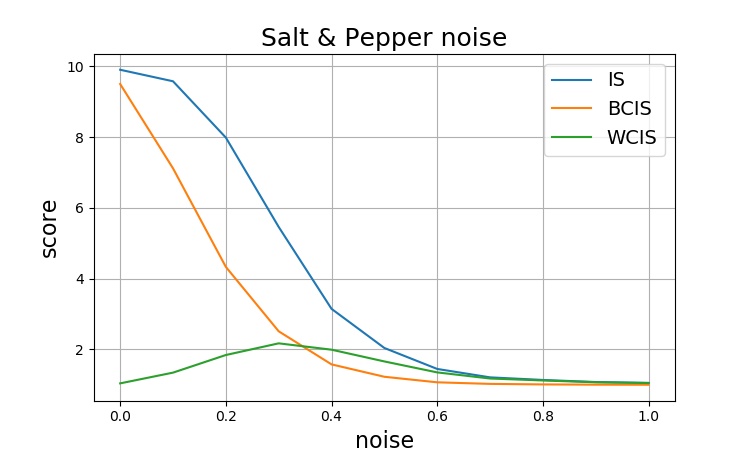} &
\includegraphics[width=0.3045\linewidth, trim={10 0 52 37}, clip]{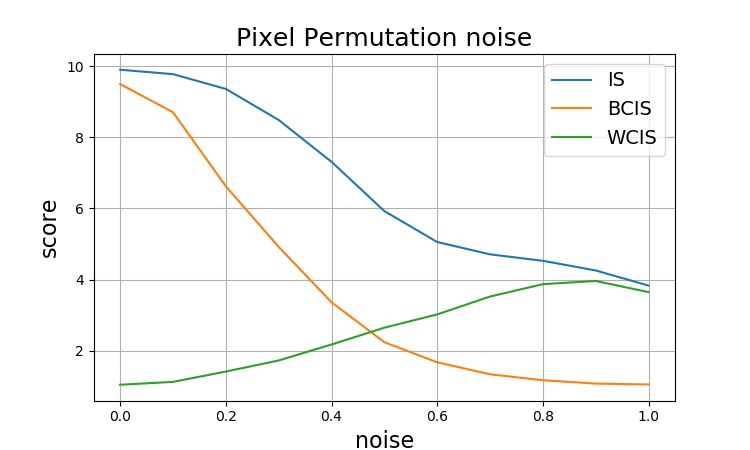} \\
\includegraphics[width=0.3045\linewidth, trim={10 0 52 37}, clip]{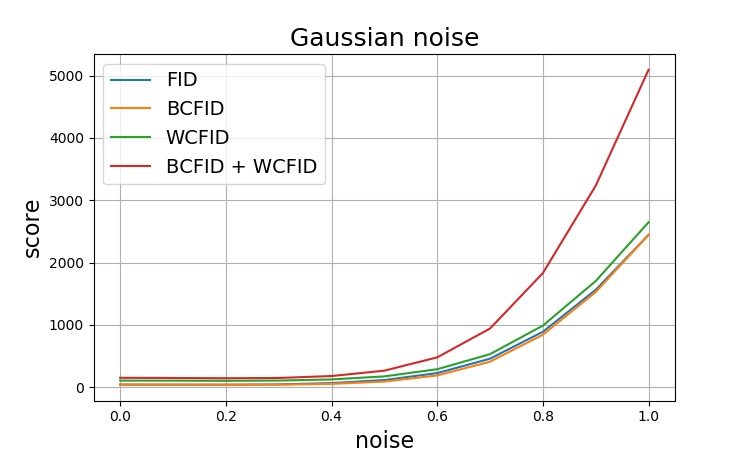} &
\includegraphics[width=0.3045\linewidth, trim={10 0 52 37}, clip]{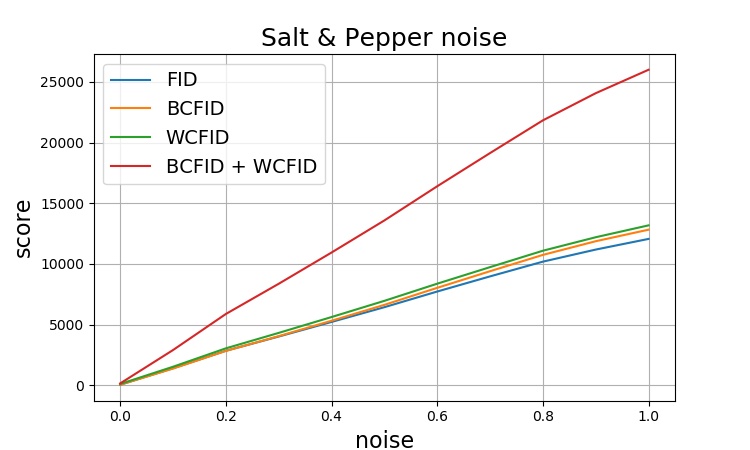} &
\includegraphics[width=0.3045\linewidth, trim={10 0 52 37}, clip]{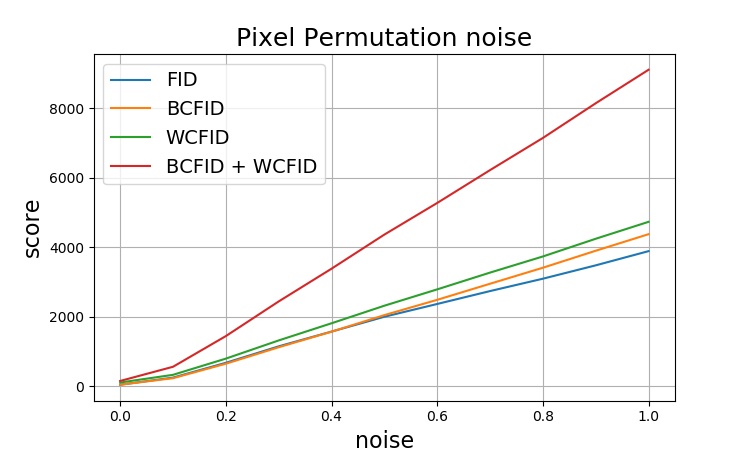} \\
(a)&(b)&(c)\\
\end{tabular}
\caption{{\bf Image noising:} The effect of various noises on the unconditional and conditional metrics over an increasing magnitude. (a) Gaussian noise. (b) Salt \& Pepper noise. (c) Random pixel permutation. (top) Example images before and after the application of each noise, (middle) The effect on the unconditional and conditional Inception Score due to varying noise intensity, (bottom) The effect on the unconditional and conditional FID due to varying noise intensity.}
\label{fig:image_noising_plot}
\end{figure*}

\begin{figure*}[t]
\centering
\begin{tabular}{ccc}
Blur & Swirl & PCA \\
\includegraphics[width=0.3045\linewidth, trim={0 0 2 0}, clip]{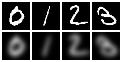} &
\includegraphics[width=0.3045\linewidth, trim={0 0 2 0}, clip]{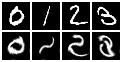} &
\includegraphics[width=0.3045\linewidth, trim={0 0 2 0}, clip]{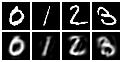} \\
\includegraphics[width=0.3045\linewidth, trim={10 0 52 40}, clip]{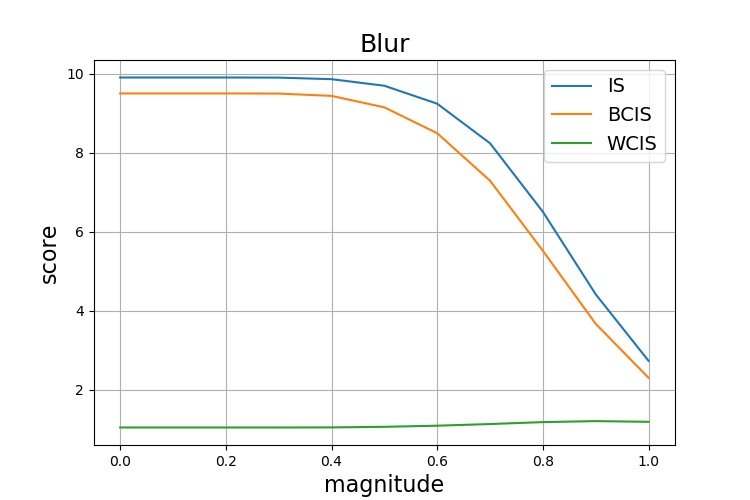} &
\includegraphics[width=0.3045\linewidth, trim={10 0 52 40}, clip]{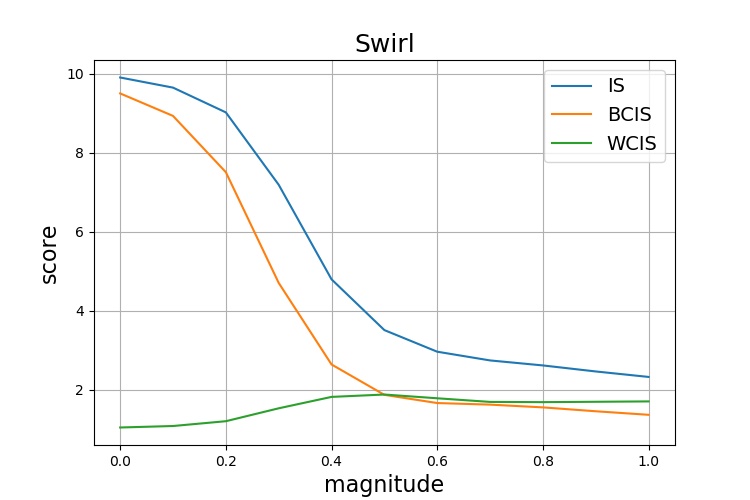} &
\includegraphics[width=0.3045\linewidth, trim={10 0 52 40}, clip]{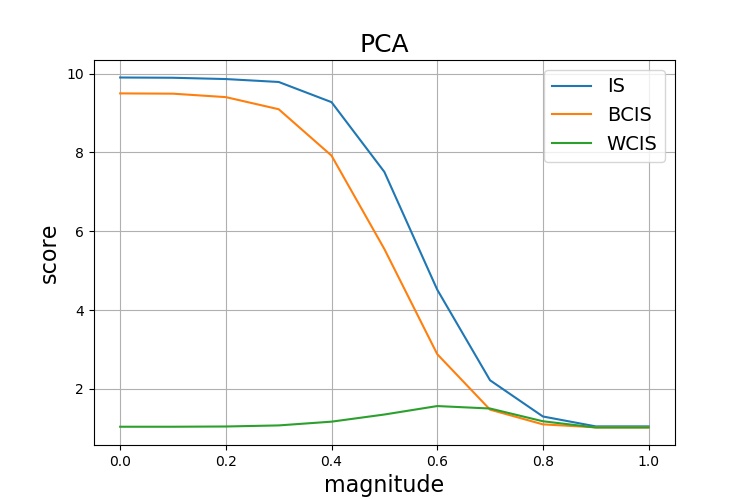} \\
\includegraphics[width=0.3045\linewidth, trim={10 0 52 40}, clip]{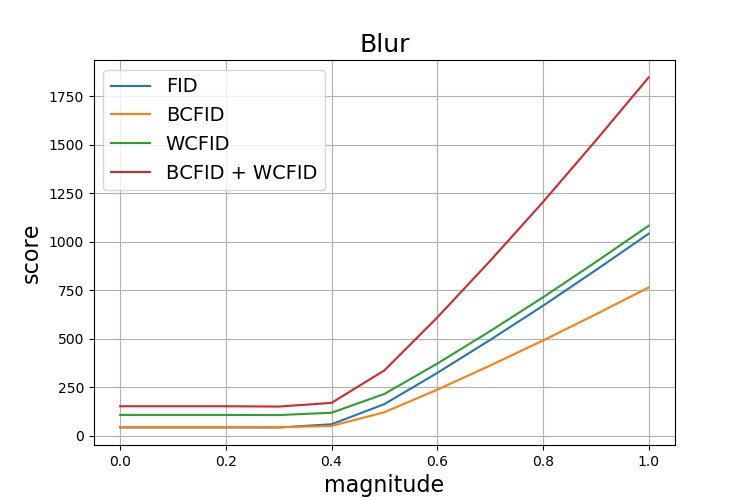} &
\includegraphics[width=0.3045\linewidth, trim={10 0 52 40}, clip]{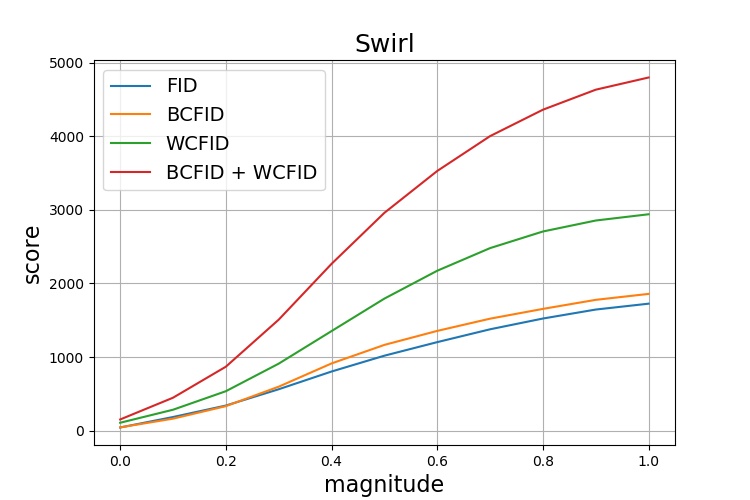} &
\includegraphics[width=0.3045\linewidth, trim={10 0 52 40}, clip]{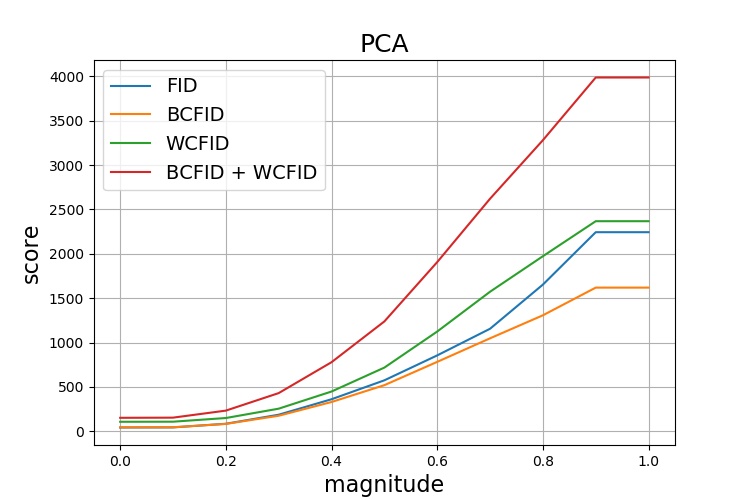} \\
(a)&(b)&(c)\\
\end{tabular}
\caption{{\color{black}{\bf Image manipulation:} The effect of various manipulations on the unconditional and conditional metrics over an increasing magnitude. (a) Gaussian blur filter. (b) Swirl effect applied on the image. (c) Dimensionality reduction with PCA. (top) Example images before and after the application of each manipulation, (middle) The effect on the unconditional and conditional Inception Score due to varying manipulation intensity, (bottom) The effect on the unconditional and conditional FID due to varying manipulation intensity.}}
\label{fig:image_manipulation_plot}
\end{figure*}

\begin{figure*}[t]
\centering
\begin{tabular}{cc}
    \begin{tabular}{c}
    \includegraphics[width=0.45\linewidth, trim={13 0 52 5}, clip]{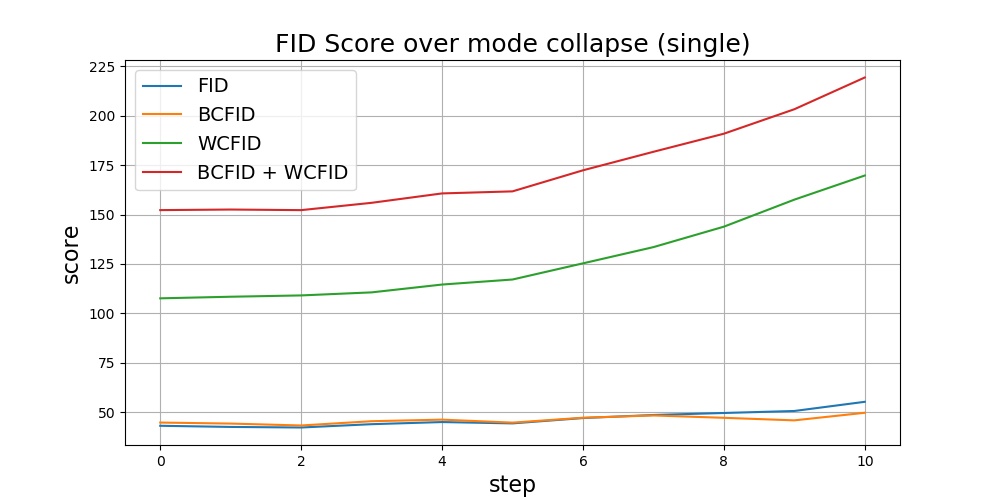} \\
    (a)\\
    \end{tabular} 
    \begin{tabular}{c}
    \includegraphics[width=0.45\linewidth, trim={13 0 52 5}, clip]{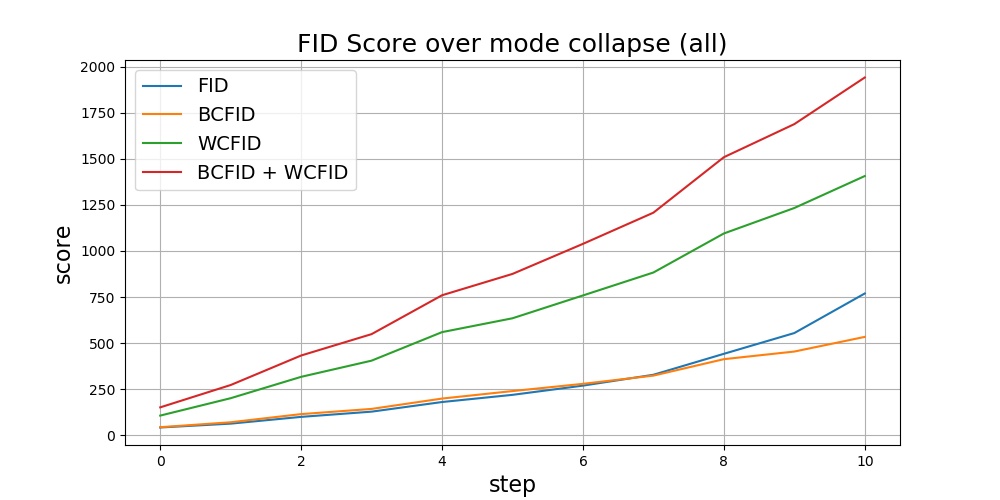} \\
    (b)\\
    \end{tabular} 
\end{tabular}
\caption{{\bf Mode collapse:} The effect of mode collapse on the unconditional and conditional IS and FID metrics over an increasing severity. (a): Gradual mode collapse on a single class. (b): Incremental full mode collapse on all classes.}
\label{fig:mode_collapse_plot}
\end{figure*}

\begin{figure*}
    \centering
    \includegraphics[width=\textwidth]{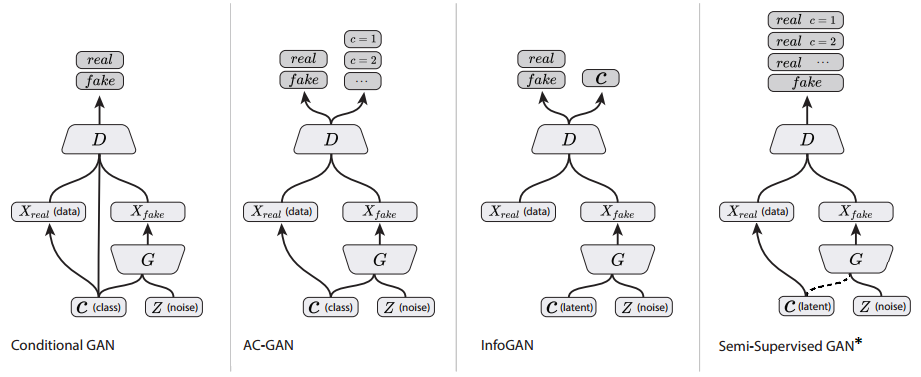}
    \caption{{\color{black}Abstract depiction of CGAN, InfoGAN, ACGAN, and our modified SGAN. The modification of SGAN is visualized with the dotted line. Figure is a modified version from https://github.com/clvrai/ACGAN-PyTorch}}
    \label{fig:models}
\end{figure*}

\section{Experiments}\label{experiments}

Our experiments employ three datasets: MNIST~\cite{lecun}, CIFAR10~\cite{cifar}, and ImageNet~\cite{russakovsky2015imagenet}. We first consider controlled simulations on MNIST to show the behavior of our metrics compared to existing unconditional metrics. Three cases are considered: noisy labels, noisy images, and mode collapse within classes. We then consider our metrics on a variety of well-established generative models and draw visual insights for the reported metric scores. Finally, a user study was held to compare the numeric results to human perception.

\noindent{\bf Evaluation procedure}
When evaluating the models, we use an equal number of randomly sampled real and generated samples for each class. For MNIST and CIFAR10, the test set was used as real samples, with $1000$ samples from each class. For ImageNet, $50$ validation samples for each class were used, for a total of $50,000$ validation samples.  

To obtain the scores of the `Real Data' in Tab.~\ref{tab:gansevaluation},~\ref{tab:biggan} (i.e., the score obtained not from generating but from the training data itself, which serves as an unofficial upper bound of the performance), an equal number of samples were taken from the train set. For instance, for MNIST, $10,000$ samples were taken from the train data ($1000$ for each class). These same samples were also used for the synthetic experiments of noise and mode collapse where they undergo various augmentations. 

For each dataset, we applied a pretrained classifier, to give class probabilities for calculating the Inception Scores, and as a feature extractor, to calculate the FID scores. For ImageNet, we used the InceptionV3~\cite{Szegedy2015RethinkingTI} architecture, as used in the original formulation of the IS~\cite{salimans2016improved} and FID~\cite{heusel2017gans}. For CIFAR10, we used the VGG-16~\cite{simonyan2014very} architecture, and for MNIST, a classifier with two convolutional blocks and two fully connected layers. The test accuracy is 99.06\% for MNIST, 85.20\% for CIFAR10 and 77.45\% for ImageNet. 

The activations of the last hidden layer (a.k.a the penultimate layer) were employed as the extracted features $f(x)$. The feature dimension is $128$ for MNIST, $512$ for CIFAR10 and $2048$ for ImageNet. 

Note that since the classification and feature extraction differs between each dataset, model scores should be compared per dataset, and not between datasets.

{\color{black}
\noindent{\bf Dealing with low-rank covariance matrices\quad}
In the unconditional setting, the estimation of the covariance matrix is done with a large set of images. However, the number of samples can be greatly reduced in the conditional setting. If the number of classes is small (e.q. MNIST, CIFAR10), the covariance matrices $\Sigma_B^R, \Sigma_B^G$ are likely to be of low rank. When the number of samples per class is small (e.q. ImageNet), this applies to $\Sigma_{W,c}^R, \Sigma_{W,c}^G$. We have found that, especially for WCFID, this can results in very unstable results.
As a compromise, instead of measuring the FID on the entire feature space, we randomly select as set of features and perform the covariance computation in a  sub-space that meets the limitations of the number of classes and the number of samples per class. We run this evaluation 100 times in order to reduce the variance caused by the sampling. For ImageNet, we randomly select $50$ features from the $2048$ feature vector, compute the FID score using these features, and finally average the FID scores of all trials to get a final FID score. For MNIST and CIFAR10 this was done with $10$ features at a time.

Because the FID scales with the number of dimensions used in the measurement, we normalize the measured FID (and also for the BCFID and WCFID) by the number of dimensions.
}

\subsection{Synthetic Experiments}
\label{sec:simulations}

\noindent{\bf Label noising} Label noising is the process of assigning random labels to some of the images, instead of their ground truth labels. This process simulates different levels of adherence to the conditional input. To maintain an equal number of images per conditioned class, instead of simply re-selecting a random class, we performed a random permutation of a subset of the images proportionally to a parameter $p \in [0,1]$. When $p=0$ no noising was applied and when $p=1$ all image labels were randomly permuted. Fig.~\ref{fig:label_noising} shows how label noising simulates decline in conditional generation performance. In Fig.~\ref{fig:label_noising}(a) each row of each subfigure represents a conditioned class, the red images highlight when the conditional generation fails. When setting $p=0$, all images are correctly generated on their conditional input and as $p$ increases, more images are incorrectly generated. 

In Fig.~\ref{fig:label_noising}(b) and (c), the IS and FID metrics and our proposed conditional variants are presented under the effect of label noising. The plots depict a number of interesting behaviors. First, the unconditional IS and FID remain constant across the experiment. That is because these metrics do not consider any conditional requirements from the generated images, and the unconditional performance has remained the same. Second, label noising has a dramatic effect on the conditional IS and FID metrics. The BCIS, which evaluates both the consistency of each condition in the target classes and the coverage of the target classes falls immediately due to the declining consistency in the conditioned images. The WCIS, on the other hand, which measures inconsistency, shows a rapid increase as a compensation of the decline of the BCIS score. All conditional components of the FID increase, since the label noise inflicts a shift in the distribution within each class and on the class averages.

\noindent{\bf Image noising}
To measure if the conditional metrics are also sensitive to unconditional degradation in quality, we applied three types of unconditional noise on the images and compared the effect on the scores. The noise was applied with increasing magnitude $p$ between $[0,1]$. We applied Gaussian noise with mean $0$ and variance $p$, salt \& pepper noise with probability $p$ per pixel, and random pixel permutation with probability $p$. 
Fig.~\ref{fig:image_noising_plot} shows the IS and FID with the conditional scores. 
For IS, the BCIS declines more rapidly than the IS, making it more sensitive to image quality. This is matched with an increase of the WCIS, which defines the gap between BCIS and IS, due to the trade-off shown in Eq.~\ref{eq:ISeqCIS}. The WCIS provides a support for the IS which gives a false sense of generation quality, best seen during pixel permutation. For FID, the conditional metrics have the same trend as the unconditional one. With the gap between the FID and the conditional FID sum increasing with the level of noise.

{\color{black}
\noindent{\bf Image manipulation\quad}
Image noise is often found in real images due to the circumstances under which the images were taken. However, it is not very prominent in generated images. Other artifacts, such as image blur, structure deformation, and compression loss, are usually more observable in generated images. In order to simulate that, we applied three different manipulations on the images and compared their effect on the scores once more. Again, the applied manipulation was done with an increasing magnitude $p \in [0,1]$. We applied Gaussian filter (blur) with $\sigma=5p$, swirl effect with a rotation of $\theta=\pi p \frac{|r-R|}{R}$ around the center ($r$ is the distance from the center for each pixel and R is the max distance in a single axis), and dimensionality reduction with PCA with $\lfloor K ^ {(1-p)} \rfloor$ components ($K=764$ is the number of pixels).

Fig.~\ref{fig:image_manipulation_plot} shows the IS and FID with the conditional scores. 
We have found the same behaviour as with `image noising' (Fig.~\ref{fig:image_noising_plot}). For IS, there is a constant decline in the BCIS, and an increase in WCIS dampens the decrease of the unconditional IS. For FID, all metrics follow the same trend and there is an increasing gap between the FID and the conditional FID sum as the effect grows in magnitude.
}

\begin{figure*}[t]
\centering
\begin{tabular}{l@{~}c@{~}c@{~}c@{~}c}
& CGAN & ACGAN & InfoGAN & SGAN \\
\includegraphics[width=0.06\textwidth]{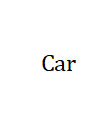} &
\includegraphics[width=0.22\textwidth]{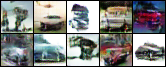} &
\includegraphics[width=0.22\textwidth]{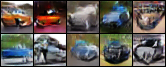} &
\includegraphics[width=0.22\textwidth]{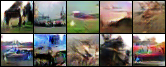} &
\includegraphics[width=0.22\textwidth]{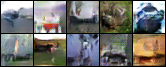} \\
\includegraphics[width=0.06\textwidth]{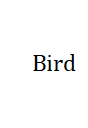} &
\includegraphics[width=0.22\textwidth]{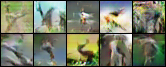} &
\includegraphics[width=0.22\textwidth]{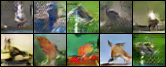} &
\includegraphics[width=0.22\textwidth]{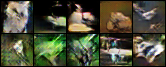} &
\includegraphics[width=0.22\textwidth]{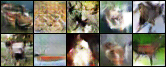} \\
\includegraphics[width=0.06\textwidth]{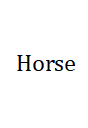} &
\includegraphics[width=0.22\textwidth]{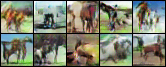} &
\includegraphics[width=0.22\textwidth]{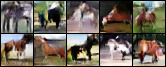} &
\includegraphics[width=0.22\textwidth]{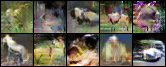} &
\includegraphics[width=0.22\textwidth]{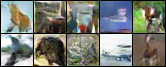} \\
\includegraphics[width=0.06\textwidth]{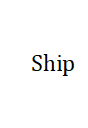} &
\includegraphics[width=0.22\textwidth]{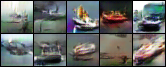} &
\includegraphics[width=0.22\textwidth]{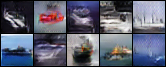} &
\includegraphics[width=0.22\textwidth]{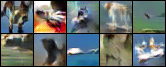} &
\includegraphics[width=0.22\textwidth]{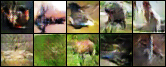} \\
&(a)&(b)&(c)&(d) \\
\end{tabular}
\caption{Illustrations for CIFAR10. Conditional generation of (a) CGAN, (b) ACGAN (c) InfoGAN and (d) SGAN.}
\label{fig:cifar}

\hfill \\

\begin{tabular}{l@{~}c@{~}c@{~}c@{~}c}
& CGAN & ACGAN & InfoGAN & SGAN \\
\includegraphics[width=0.06\textwidth]{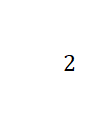} &
\includegraphics[width=0.22\textwidth]{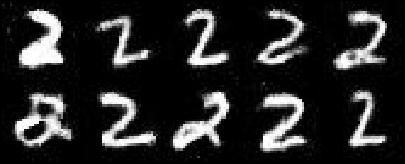} &
\includegraphics[width=0.22\textwidth]{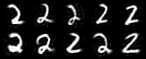} &
\includegraphics[width=0.22\textwidth]{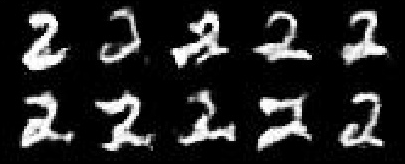} &
\includegraphics[width=0.22\textwidth]{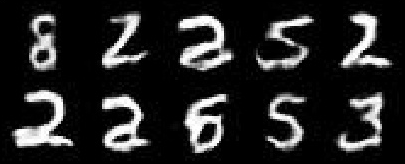} \\
\includegraphics[width=0.06\textwidth]{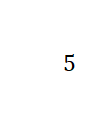} &
\includegraphics[width=0.22\textwidth]{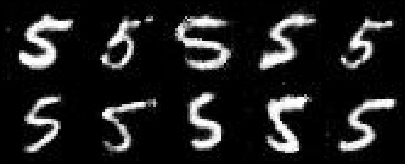} &
\includegraphics[width=0.22\textwidth]{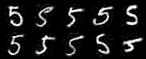} &
\includegraphics[width=0.22\textwidth]{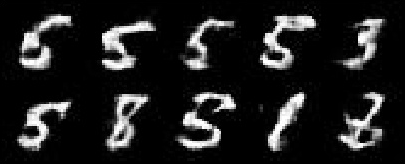} &
\includegraphics[width=0.22\textwidth]{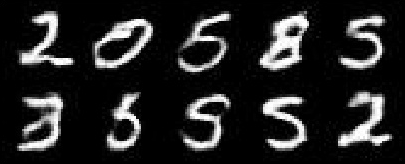} \\
\includegraphics[width=0.06\textwidth]{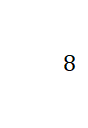} &
\includegraphics[width=0.22\textwidth]{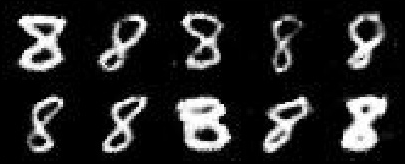} &
\includegraphics[width=0.22\textwidth]{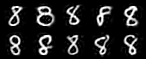} &
\includegraphics[width=0.22\textwidth]{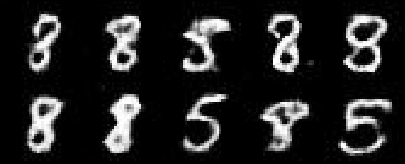} &
\includegraphics[width=0.22\textwidth]{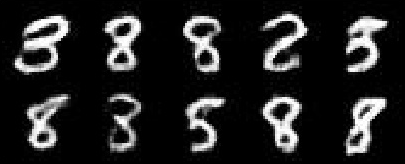} \\
\includegraphics[width=0.06\textwidth]{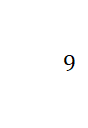} &
\includegraphics[width=0.22\textwidth]{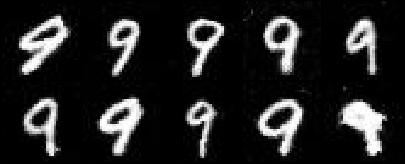} &
\includegraphics[width=0.22\textwidth]{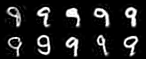} &
\includegraphics[width=0.22\textwidth]{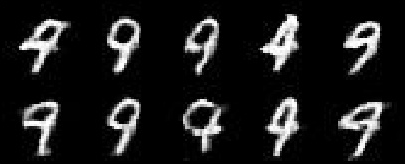} &
\includegraphics[width=0.22\textwidth]{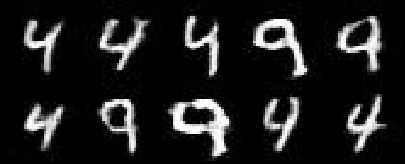} \\
&(a)&(b)&(c)&(d) \\
\end{tabular}
\caption{Illustrations for MNIST. Conditional generation of (a) CGAN, (b) ACGAN (c) InfoGAN and (d) SGAN.}
\label{fig:mnist}
\end{figure*}

\begin{table*}[t]
\caption{Unconditional and conditional metrics on CIFAR10 and MNIST for different conditional GANs. 
$\downarrow$ indicates that a lower value is better and $\uparrow$ otherwise.}
\label{tab:gansevaluation}
\centering
\begin{tabular}{l@{~~}lc@{~}c@{~}cc@{~}c@{~}ccc@{~}c@{~}c}
\toprule
& & \multicolumn{7}{c}{Evaluation metrics} & \multicolumn{3}{c}{User study}\\
\cmidrule(lr){3-9}
\cmidrule(lr){10-12}
&        & FID $\downarrow$    & WCFID 	$\downarrow$ & BCFID 	$\downarrow$ & IS 	$\uparrow$   & WCIS 	$\downarrow$ & BCIS 	$\uparrow$ & Acc 	$\uparrow $ & Quality	$\uparrow $ & Diversity 	$\uparrow $ & Class	$\uparrow $\\  
\midrule
\parbox[t]{2mm}{\multirow{5}{*}{\rotatebox[origin=c]{90}{CIFAR10~}}}
& Real Data             & 0.02  & 0.09  & 0.04  & 8.45  & 1.14  & 7.41  & 91.15 & - & - & - \\
\cmidrule(lr){2-12}

& CGAN                  & 4.31  & 7.14  & 6.30  & 7.10  & 1.22  & 5.78  & 69.51 & 7.6 & 7.5 & 7.0 \\
& ACGAN                 & 4.23  & 6.84  & 5.25  & 7.01  & 1.10  & 6.36  & 75.43 & 7.5 & 7.4 & 8.0 \\
& InfoGAN               & 5.17  & 14.59 & 10.13 & 6.52  & 2.01  & 3.25  & 65.69 & 4.1 & 6.9 & 4.2 \\
& SGAN                  & 4.46  & 17.07 & 12.34 & 6.84  & 2.52  & 2.71  & 11.27 & 5.3 & 7.2 & 2.5 \\

\midrule
\parbox[t]{2mm}{\multirow{5}{*}{\rotatebox[origin=c]{90}{MNIST~}}}
& Real Data             & 19.86 & 35.99     & 34.17     & 9.86  & 1.04  & 9.52  & 99.61 & - & - & - \\
\cmidrule(lr){2-12}
& CGAN                  & 36.67 & 98.48     & 25.70     & 9.87  & 1.06  & 9.31  & 98.90 & 8.3 & 8.8 & 7.6  \\
& ACGAN                 & 30.13 & 91.21     & 25.95     & 9.74  & 1.09  & 8.93  & 98.30 & 8.5 & 9.0 & 9.1 \\
& InfoGAN               & 76.73 & 321.56    & 93.51     & 9.38  & 1.33  & 7.03  & 89.83 & 5.3 & 6.0 & 6.3 \\
& SGAN                  & 69.34 & 609.48    & 289.42    & 8.87  & 2.03  & 4.37  & 73.34 & 6.0 & 3.5 & 5.3 \\

\bottomrule
\end{tabular}
\end{table*}

\noindent{\bf Mode collapse} 
Mode collapse occurs when the model fails to generalize on the distribution of the target dataset and collapses to represent only a portion of the distribution. It is a common failure of generative models, which occurs when the model $G$ generates similar images for many different initial priors $z$. In the conditional setting, the collapse can be more specific and occur only within a specific class.

{\color{black}
To simulate mode collapse, we performed the evaluation 11 times (steps). In each step, we sub-sampled the subset of the collapsed classes so that the remained set is 2/3 of its previous size. This leads to the final step having a pool size of less than 2\% of the original size. We then randomly selected 100 images from each class and measured the scores on the selected images.
}

Fig.~\ref{fig:mode_collapse_plot} shows how the unconditional and conditional FID metrics react to the collapse.
(a) shows a single class collapse in where in each step the diversity in that class gradually declines. (b) shows all of the classes fully collapse one by one at each step. Our metric is more sensitive to mode collapse, both when it occurs in a single class or in multiple classes. {\color{black}As can be seen, the mode collapse is the most noticable in WCFID, and is less detectable in BCFID. This is because the collapse reduces the variance inside each class, but does not change the class average very much. This results is well aligned with the observation in Sec.~\ref{CFID} that BCFID measures the coverage in the representation of the classes, while WCFID is more sensitive the the diversity inside each class. We have found the FID to be as sensitive to mode collapse as BCFID, which means that it is less sensitive to mode collapse than WCFID when it occurs in specified classes only.}

No evaluation on the unconditional and conditional IS was performed in this setting, since they both cannot detect mode collapse.

\subsection{Model Comparison}
We next evaluate the performance of various pretrained conditional GAN models on different datasets. 
In Tab.~\ref{tab:gansevaluation}, for CIFAR10 and MNIST, we consider CGAN~\cite{mirza2014conditional}, ACGAN~\cite{odena2017conditional}, InfoGAN~\cite{chen2016infogan}, and SGAN~\cite{odena2016semi}.
Note that for SGAN, the generator is not class conditioned, and so we modified the generator to accept both noise and class label as input, and the adversarial loss was applied on the conditioned class. {\color{black}The conceptual differences between the models can be seen in Fig.~\ref{fig:models}.}

For conditional generation, there are four extreme cases: (i) good unconditional and good conditional generation, (ii) bad unconditional and bad conditional generation, (iii) good unconditional and bad conditional generation, (iv) bad unconditional and good conditional generation. We argue that the fourth scenario is impossible since the conditional generation metrics always present a more critical evaluation (i.e. a lower bound in IS and upper bound in FID) than the unconditional metric. Therefore bad unconditional generation always leads to bad conditional generation as well. Cases (i) and (ii) are the more trivial cases where the model is either good or bad on both tasks. Case (iii) tells a scenario where the unconditional generation is good but the conditional requirement failed. We will now inspect each model and identify under which scenarios it falls.

{\color{black}The analysis is done by following the results in Tab.~\ref{tab:gansevaluation}. Additionally, Fig.~\ref{fig:cifar},\ref{fig:mnist} show examples of the generation for each dataset for selected classes. 
In generation, relatively to each other, ACGAN and CGAN lie under case (i), InfoGAN under (ii), and SGAN under (iii) for CIFAR10 and (ii) for MNIST.
Both CGAN and ACGAN performed better than InfoGAN and SGAN in all metrics. CGAN and ACGAN performed very similar on CIFAR10 and ACGAN slightly better on MNIST.
Both InfoGAN and SGAN performed poorly compared to the other two models. However, we observed that SGAN performed better than InfoGAN on the unconditional metrics, but worse on the conditional ones. This shows that the condition applied by the SGAN discriminator is less efficient than that of CGAN or ACGAN.

The results in the tables can be understood from Fig.~\ref{fig:cifar},\ref{fig:mnist}. Both CGAN and ACGAN succeed reasonably well in generation and representing the correct class, on both datasets. InfoGAN generates some images that are from the correct class, but wrong classes are often generated as well. Finally, SGAN fails to do any consistent conditioning on CIFAR10 and has many wrong classes in MNIST as well.

The lower performance of InfoGAN is not surprising, because it operates in an unsupervised setting. However, it is surprising to see that the condition applied by its generator (which is similar in design to that of ACGAN) allows it to better perform than SGAN in conditional generation, despite the latter being trained in a supervised fashion.

The experiments on CIFAR10 highlight the importance of conditional evaluation metrics. SGAN performs quite similarly to CGAN and ACGAN in both FID and IS. However, it is in a clear disadvantage when comparing the conditional expansions. 
InfoGAN also highlights the problem of using accuracy as a measure. The accuracy received for the generated images in CIFAR10 is relatively high, and not far behind CGAN and ACGAN. However, the conditional metrics all show that the evaluation based on accuracy is too optimistic.

Another observation we have found in Tab.~\ref{tab:gansevaluation}, is that the IS was not very informative. On both CIFAR10 and MNIST, all models received a relatively high score, even though other metrics showed otherwise. This is not the case for the BCIS metric, which is sensitive in detecting models that performed badly in the conditional sense.

\noindent {\bf User study}
To see how the metrics translate to human perception, we performed a user study for both MNIST and CIFAR10. The user study was performed on 20 participants with knowledge in this field. The participants were not aware of the purpose of the study and did not know which model they were evaluating. The participants were asked to grade the `quality', `diversity' and `class relation' of the generated images between 1 (low) and 10 (high), for each model separately. The results in Tab.~\ref{tab:gansevaluation} show that ACGAN got higher scores on both datasets for conditional generation, with CGAN not far behind, and SGAN performing the worst. This is aligned with the results of the conditional metrics in our experiments. The study also shows that without considering the condition, the overall quality of ACGAN and CGAN is the same, and the quality of InfoGAN is the worst. This agrees with the conclusion of the quantitative evaluation as well.
}

\begin{table}[t]
\caption{Unconditional and conditional metrics on ImageNet for BigGAN. 
$\downarrow$ indicates that a lower value is better and $\uparrow$ otherwise.  }
\label{tab:biggan}
\centering
\begin{tabular}{@{~}l@{~}l@{~}c@{~}c@{~}c@{~}c@{~}c@{~}c@{~}c}
\toprule
&        & \multicolumn{3}{c}{FID} & \multicolumn{3}{c}{IS} &   \\  
\cmidrule(lr){3-5}
\cmidrule(lr){6-8}
&        & -$\downarrow$    & WC$\downarrow$ & BC$\downarrow$ & -$\uparrow$   & WC$\downarrow$ & BC$\uparrow$ & Acc$\uparrow $ \\  
\midrule
& Real Data   & 0.11 & 4.15 & 0.11 & 602.61 & 3.01 & 201.36 & 77.45 \\
\cmidrule(lr){2-9}
& BigGAN        & 0.46 & 6.33  & 0.43 & 363.91  & 5.04 & 72.15 & 51.66 \\
& BigGAN@200K   & 1.06 & 9.17 & 0.94 & 257.38  & 7.62 & 33.71 & 27.70 \\
& BigGAN@100K   & 1.27 & 11.36 & 1.51 & 112.29  & 9.21 & 12.19 & 18.87 \\
\bottomrule
\end{tabular}
\end{table}

\begin{figure}[t]
\centering
\includegraphics[width=0.48\textwidth, trim={50 0 60 0}, clip]{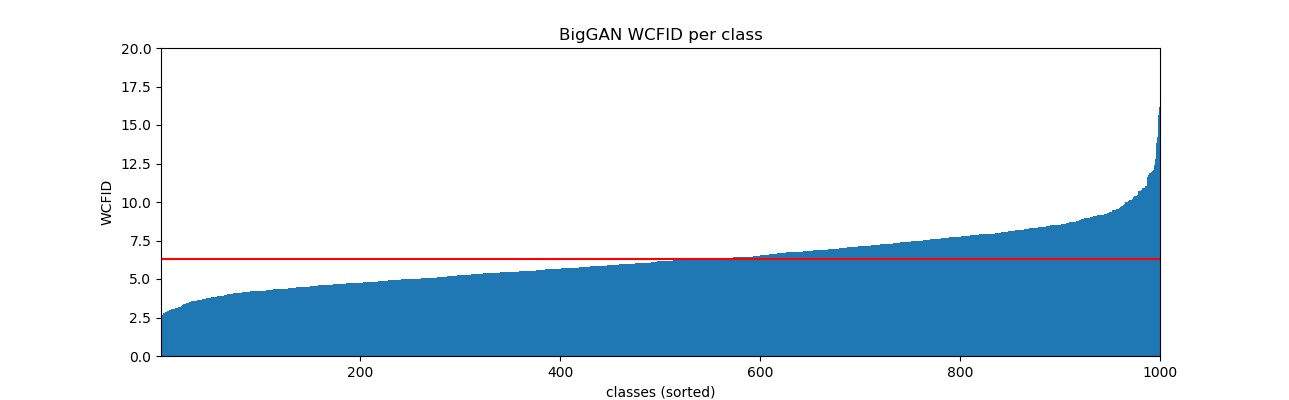} \\
(a) \\
\includegraphics[width=0.48\textwidth, trim={50 0 60 0}, clip]{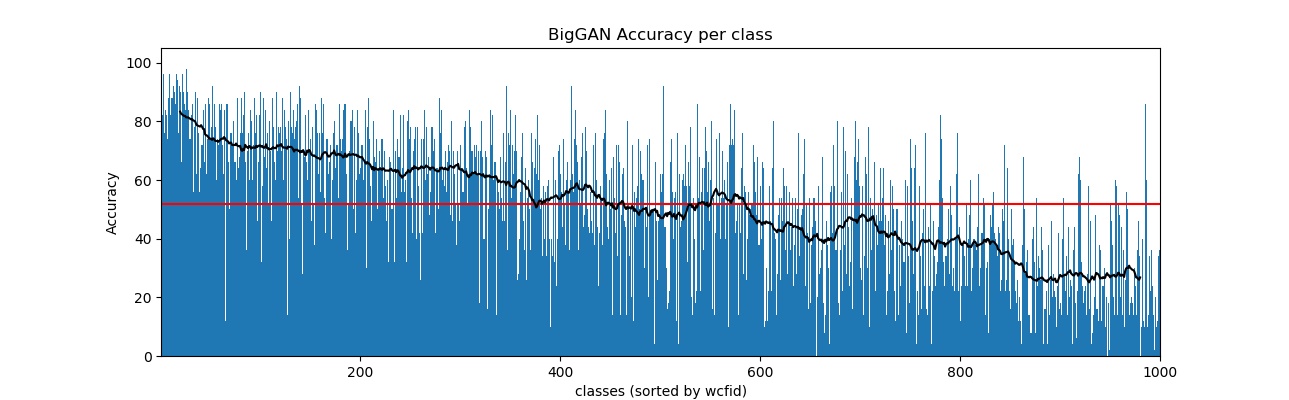} \\
(b) \\
\includegraphics[width=0.48\textwidth, trim={50 0 60 0}, clip]{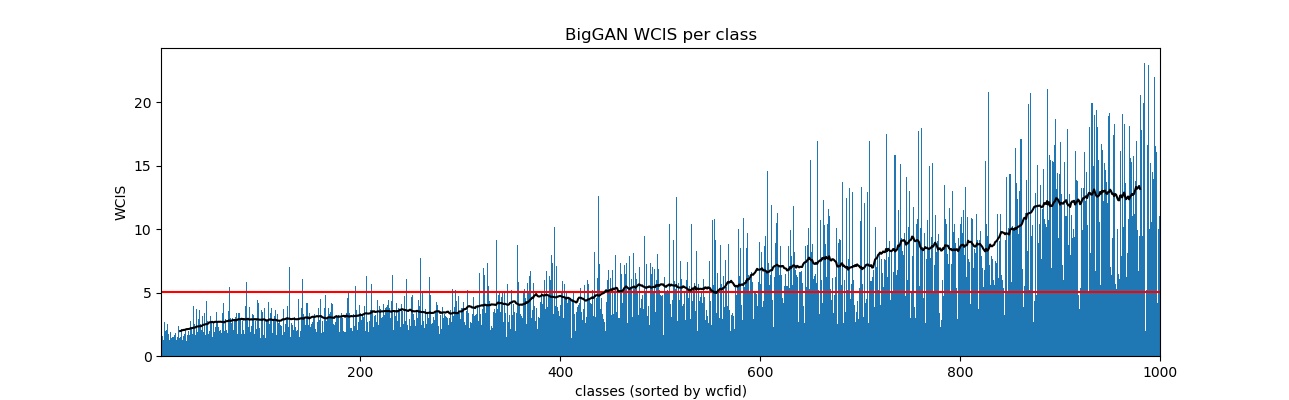} \\
(c) \\
\caption{Not all classes perform equally, and the classes are not ranked similarly for different metrics. (a) WCFID, (b) accuracy, and (c) WCIS, per class for BigGAN. The average score is shown in red. A windowed average over 50 bins is shown in black.}
\label{fig:biggan1}
\end{figure}

\begin{figure*}[t]
\centering
\begin{tabular}{cc}
\includegraphics[width=0.47\textwidth, trim={0 30 90 20}, clip]{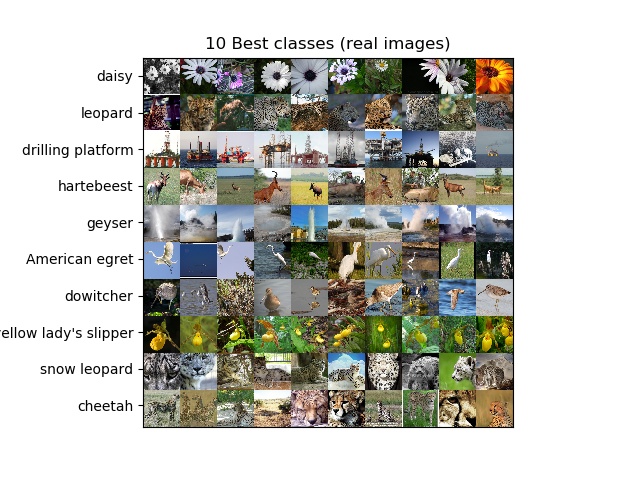} &
\includegraphics[width=0.47\textwidth, trim={90 30 0 20}, clip]{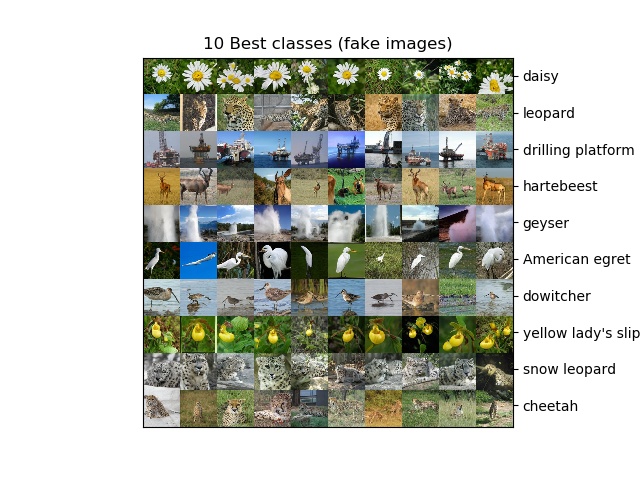} \\
(a)&(b) \\
\includegraphics[width=0.47\textwidth, trim={0 30 90 20}, clip]{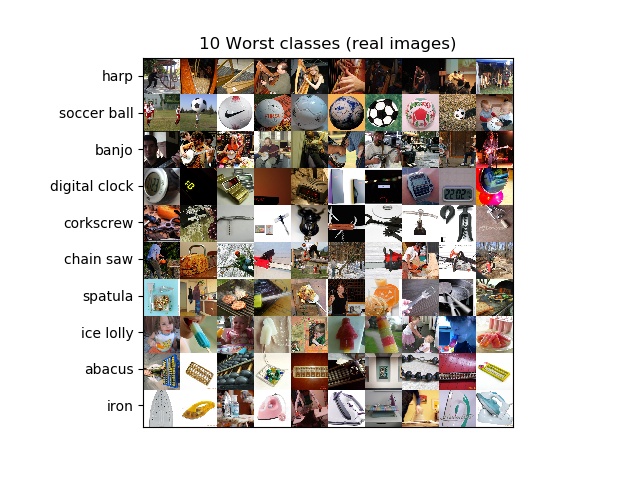} &
\includegraphics[width=0.47\textwidth, trim={90 30 0 20}, clip]{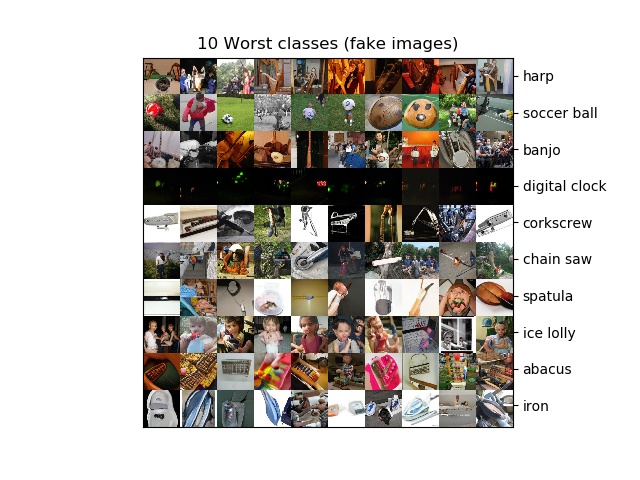} \\
(c)&(d) \\
\end{tabular}
\caption{BigGAN images of best and worst classes in terms of WCFID. (a),(b) real and fake images for classes with the highest WCFID. (c),(d) real and fake images for classes with the lowest WCFID.}
\label{fig:biggan1_wcfid_images}
\end{figure*}

\begin{figure*}[t]
\centering
\begin{tabular}{cc}
\includegraphics[width=0.47\textwidth, trim={0 30 90 20}, clip]{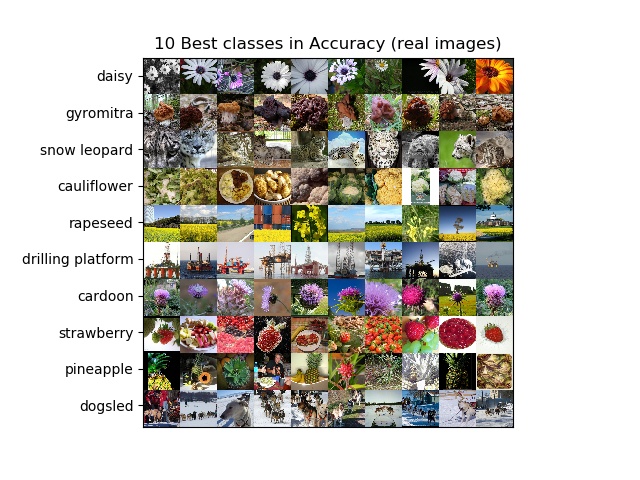} &
\includegraphics[width=0.47\textwidth, trim={90 30 0 20}, clip]{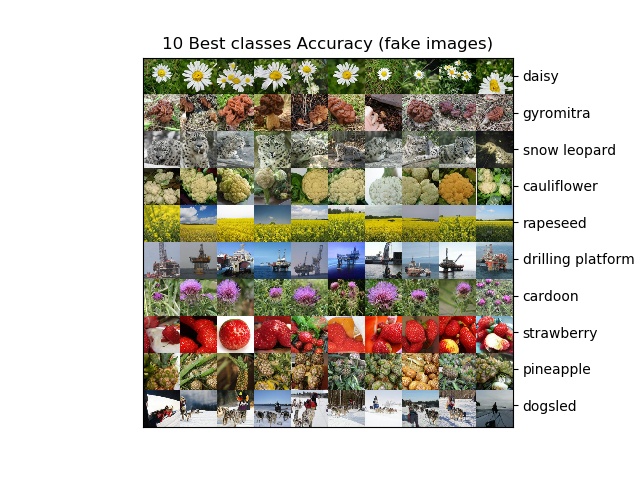} \\
(a)&(b) \\
\includegraphics[width=0.47\textwidth, trim={0 30 90 20}, clip]{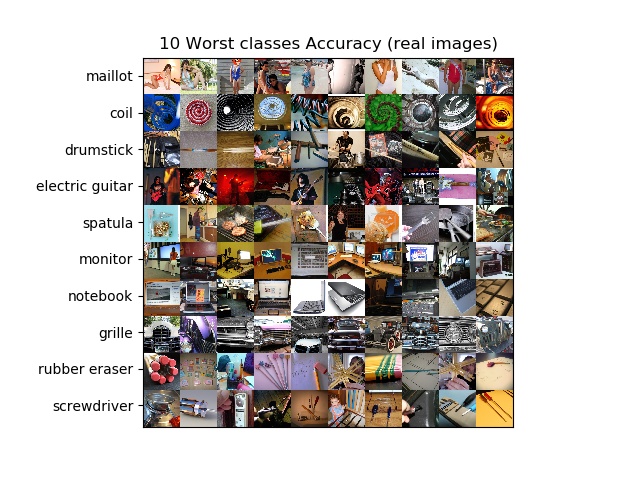} &
\includegraphics[width=0.47\textwidth, trim={90 30 0 20}, clip]{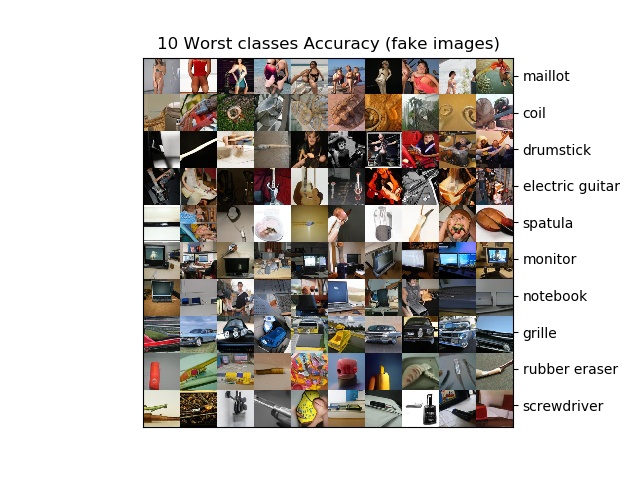} \\
(c)&(d) \\
\end{tabular}
\caption{BigGAN images of best and worst classes in terms of Accuracy. (a),(b) real and fake images for classes with the highest accuracy. (c),(d) real and fake images for classes with the lowest accuracy.}
\label{fig:biggan1_accuracy_images}
\end{figure*}

\subsection{BigGAN In-Depth Analysis}\label{biggan}
BigGAN~\cite{brock2018large} is a state of the art image generation model on the ImageNet~\cite{russakovsky2015imagenet} dataset. In this section we evaluate BigGAN with our metrics and use them to perform an in-depth analysis of its conditional generation capabilities.

{\color{black}
BigGAN's performance on the various metrics can be seen in Tab.~\ref{tab:biggan}. We measured the performance of both the fully trained BigGAN and intermediate stages\footnote{\url{https://github.com/ajbrock/BigGAN-PyTorch}} of the training. Note that for FID, the score is different than in the original paper since we normalized the score by the size of the feature vector, as described in the beginning of Sec.~\ref{experiments}. The results show that BigGAN improves on all metrics during training. They also show, that compared to advanced stages of the training (100K, 200K), the performance is getting close to that of real images from ImageNet.
}

A closer inspection shows a variance in generation quality of the model for the different classes. Fig.~\ref{fig:biggan1}(a) shows us that not all classes have the same WCFID and, instead, some classes are better represented.
Which classes are better than others, can serve as a useful insight for fine-tuning a trained model to concentrate on the worst represented classes, or to compare between various trained generative models.

Fig.~\ref{fig:biggan1_wcfid_images} shows the 10 best and worst classes represented in terms of WCFID.
The WCFID metric has a strong correlation with the quality of the class. The images from classes with the best scores are of high quality and resemblance to the real images. The images for classes with the worst scores do not resemble their target class. In addition, as evident from the experiment, several classes (for example, 'digital clock') have a high WCFID that is due to mode collapse. 

To validate that the accuracy score cannot deliver these insights, we present in Fig.~\ref{fig:biggan1}(b) the accuracy for each class, sorted according to their WCFID (same order as (a)). Similarly to WCFID, not all classes have the same score. However,we can observe that the accuracy scores for each class are only partly (inversely) correlated with the performance in WCFID. 
In order to try and understand the difference between the scores in WCFID and accuracy, Fig.~\ref{fig:biggan1_accuracy_images} shows the best and worst classes in terms of accuracy. Some classes were placed in the top 10 in both metrics (WCFID and accuracy), but others were not equally ranked. When looking at the worst ranked classes, we notice that the low rank in accuracy does not always correlate with a low quality or diversity. For example, 'notebook' and 'monitor' were both ranked at the bottom when considering the accuracy, but looked not as bad as the worst classes in WCFID. We observe that these classes were ranked low not because they were poorly generated, but because it is hard to tell them apart. 

{\color{black}
The same comparison is performed once more in Fig.~\ref{fig:biggan1}(b) for WCIS. Similarly to accuracy, the WCIS does not rank the classes in the same order as WCFID. However, by following the smoothed line over 50 neighboring bind, we conclude that the WCIS is more similar to WCFID than the accuracy. Especially for the classes that are ranked high on WCFID (low score), we have found a high correlation with the classes that are ranked high on WCIS (low score).

Because the Inception Score is ranking the classes by their likeness to their respective class, and WCFID ranks classes by their diversity, we believe it is a good practice to analyze models with both methods.
}

{\color{black}
\section{Discussion}
This work tackled one aspect of conditional generation, the class conditional one. However, conditional generation can come in other forms, such as continuous labels, multi-label, or even text. It is not clear how to expand the existing methods to these settings, and it is also not certain that one method will be good for all settings. It is reasonable, that as more effort will be put on defining the criteria to evaluate conditional generation, new solutions will be able to support other settings as well.

In this work, the formulation of the conditional metrics applies a weighted average on both the classes and the instances of each class.
When the classes are balanced, which is often the case in computer vision datasets, the difference between applying a weighted or a simple average is negligible. However, in many real world scenarios, the class prevalence is unbalanced. In this case, applying a simple or a weighted average will have a strong effect on the final score. On the one hand, applying a weighted average will evaluate the classes proportionally to their prevalence, but might ignore low performance on rare classes. On the other hand, applying a simple weight can magnify the poor performance of a small set of cases, which will result in a poor overall performance. Either approaches are valid, depending on the objective, but different results will be observed.
}

\section{Conclusions}
We presented two new evaluation procedures for class-conditional image generation based on well established metrics for unconditional generation.
These metrics are easy to implement and can be used to compare models from different architectures and to inspect and select the best model during training. 
The proposed metrics are supported by theoretical analysis and a number of experiments. Our metrics are shown to be beneficial in comparing trained models and gaining significant insights when developing models.

\begin{acknowledgements}
This project has received funding from the European Research Council (ERC) under the European Unions Horizon
2020 research and innovation programme (grant ERC CoG 725974).

The contribution of the first author is part of a Ph.D. thesis research conducted at Tel Aviv University.
\end{acknowledgements}

%
%

\bibliographystyle{spmpsci}      
\bibliography{references}   


\clearpage
\onecolumn
\appendix
\setcounter{theorem}{0}

\section{Proofs of the Main Results}\label{sec:proofs}

\begin{lemma}\label{lem:inception} Let $C\sim \mathcal{D}_C$ and $Z\sim \mathcal{D}_Z$ be two independent random variable. Let $X = G(Z,C)$ for a continuous generator function $G$ and let $Y$ be a discrete random variable distributed by $p(y|X)$. Then,
\begin{equation}
\begin{split}
IS(X;Y) = \exp\{I(X;Y)\}
\end{split}
\end{equation}
\end{lemma}

\begin{proof} We consider that:
\begin{equation}
\begin{split}
    IS(X;Y) &= \exp \left\{\mathbb{E}_{x\sim \mathcal{D}_G}[\KL(p_G(y|x) \| p_G(y))] \right\} = \exp\left\{\int_x{p_G(x)\sum_y{p_G(y|x)\cdot \log{\frac{p_G(y|x)}{p_G(y)}}}} \textnormal{d}x\right\} \\
    &= \exp\left\{\int_x{\sum_y{p_G(x,y)\cdot \log{\frac{p_G(x,y)}{p_G(y)\cdot p_G(x)}}}} \textnormal{d}x\right\} \\
    &= \exp\{I(X;Y)\}
\end{split}
\end{equation}
\end{proof}

\begin{theorem} 
Let $C\sim \mathcal{D}_C$ and $Z\sim \mathcal{D}_Z$ be two independent random variable. Let $X = G(Z,C)$ for a continuous generator function $G$ and let $Y$ be a discrete random variable distributed by $p(y|X)$. Then,
\begin{equation}
IS(X;Y) = BCIS(X;Y) \cdot WCIS(X;Y)
\end{equation}
\end{theorem}

\begin{proof}
By Lem.~\ref{lem:inception}, the Inception Score can be represented as $IS(X;Y) = \exp\{I(X;Y)\}$, and by definition, we have: $BCIS(X;Y) = IS(C;Y)$. Next, we would like to represent $I(C;Y)$ in terms of $I(X;Y)$ and $\mathbb{E}_{c}[I(X_c;Y_c)]$. First, by marginalizing with respect to $X|C$, we have:
\begin{equation}
\begin{aligned}
I(C;Y)  =& \sum_{c} p(c) \sum_y p_G(y|c) \cdot \log\frac{p_G(y|c)}{p_G(y)}  \\
            =& \sum_c p(c) \sum_y \left( \int_{x} {\frac{p_G(x,c) \cdot p_G(y|x,c)}{p(c)}} \textnormal{ d}x \right)\cdot \log \frac{p_G(y|c)}{p_G(y)}  \\
\end{aligned}
\end{equation}
Since $Y$ is independent of $C$ given $X$, we have: $p_G(y|x,c) = p_G(y|x)$. Hence,
\begin{equation}
\begin{aligned}
I(C;Y)  =& \sum_c \sum_y \left( \int_{x} p_G(x,c) \cdot p_G(y|x) \textnormal{ d}x \right)\cdot \log \frac{p_G(y|c)}{p_G(y)} \\
=& \sum_c \left(\int_x \sum_y
        p_G(x,c)\cdot p_G(y|x) \cdot \log\left(\frac{p_G(y|c)}{p_G(y)} \cdot \frac{p_G(y|x)}{p_G(y|x)}\right) \textnormal{ d}x \right) \\
=& \sum_c \int_x \sum_y
        p_G(x,c)\cdot p_G(y|x) \cdot \log \frac{p_G(y|x)}{p_G(y)} \textnormal{ d}x  + \sum_c \int_x \sum_y
        p_G(x,c)\cdot p_G(y|x) \cdot \log \frac{p_G(y|c)}{p_G(y|x)} \textnormal{ d}x \\
\end{aligned}
\end{equation}
We consider that $p_G(x)=\sum_c p_G(x,c)$. Therefore, we have:
\begin{equation}
\begin{aligned}
I(C;Y)  =& \int_x \sum_y p_G(x)\cdot p_G(y|x) \cdot \log \frac{p_G(y|x)}{p_G(y)} \textnormal{ d}x  + \sum_c \int_x \sum_y
    p_G(x,c)\cdot p_G(y|x) \cdot \log \frac{p_G(y|c)}{p_G(y|x)} \textnormal{ d}x \\
=& I(X;Y) + \sum_c \int_x \sum_y
    p_G(x,c)\cdot p_G(y|x) \cdot \log \frac{p_G(y|c)}{p_G(y|x)} \textnormal{ d}x \\
=& I(X;Y) - \sum_c \int_x \sum_y
    p_G(x,c)\cdot p_G(y|x) \cdot \log \frac{p_G(y|x)}{p_G(y|c)}  \textnormal{ d}x \\ 
=& I(X;Y) - \sum_c p(c) \int_x \sum_y
    p_G(x|c)\cdot p_G(y|x) \cdot \log \frac{p_G(y|x)}{p_G(y|c)}  \textnormal{ d}x \\
=& I(X;Y) - \sum_c{p(c) \cdot I(X_c;Y_c)} \\
=& I(X;Y) - \mathbb{E}_{c}[I(X_c;Y_c)] \\
\end{aligned}
\end{equation}
Finally, we conclude that:
\begin{equation}
\begin{aligned}
BCIS(X;Y) &= \exp\{I(C;Y)\} \\
&= \exp\{I(X;Y) - \mathbb{E}_{c}[I(X_c;Y_c)]\} \\
&= \frac{\exp\{I(X;Y)\}}{\exp\{\mathbb{E}_{c}[I(X_c;Y_c)]\}} = \frac{IS(X;Y)}{WCIS(X;Y)}
\end{aligned}
\end{equation}
\end{proof}

\begin{theorem} Let $\mathcal{D}_R$ and $\mathcal{D}_G$ be the distributions of real and generated samples. Then,
\begin{equation}
FID(\mathcal{D}_R,\mathcal{D}_G) \leq BCFID(\mathcal{D}_R,\mathcal{D}_G) + WCFID(\mathcal{D}_R,\mathcal{D}_G)
\end{equation}
and the bound is tight. 
\end{theorem}

\begin{proof}
First, we recall the definitions of the FID and BCFID measures:
\begin{equation}
FID(\mathcal{D}_G,\mathcal{D}_R) = \|\mu^R - \mu^G\|^2 + \Tr(\Sigma^R + \Sigma^G - 2(\Sigma^R\Sigma^G)^\frac{1}{2})
\end{equation}
and
\begin{equation}
BCFID(\mathcal{D}_G,\mathcal{D}_R) = \|\mu^R_B - \mu^G_B\|^2 + \Tr(\Sigma^R_B + \Sigma^G_B - 2(\Sigma^R_B\Sigma^G_B)^\frac{1}{2})
\end{equation}
We notice that $\mu^R_B = \mu^R$ and $\mu^G_B = \mu^G$. Hence, the only difference between the two quantities arises from the second terms. 

Next, we would like to develop the formulation of $\Sigma^E_{W,c}$ for $E\in \{R,G\}$:
\begin{equation}
\begin{aligned}
\Sigma^E_{W,c}
=& \mathbb{E}_{x \sim X^E_c}\left[(f(x) - \mu^E_c)(f(x) - \mu^E_c)^T\right] \\
=& \int_{x} p_E(x|c)\cdot (f(x) - \mu^E_c)(f(x) - \mu^E_c)^T \;\textnormal{d}x \\
=& \int_{x} p_E(x|c) \cdot (f(x) - \mu^E + \mu^E - \mu^E_c)(f(x) - \mu^E + \mu^E - \mu^E_c)^T \;\textnormal{d}x \\
=& \int_{x} p_E(x|c) \cdot \left[\left((f(x) - \mu^E) + (\mu^E - \mu^E_c)\right)\left((f(x) - \mu^E) + (\mu^E - \mu^E_c)\right)^T\right] \;\textnormal{d}x \\
=& (\mu^E - \mu^E_c)(\mu^E - \mu^E_c)^T \\
&+ \int_{x} p_E(x|c) \cdot (f(x) - \mu^E)(f(x) - \mu^E)^T \; \textnormal{d}x + \int_{x} p_E(x|c) \cdot [(f(x) - \mu^E)(\mu^E - \mu^E_c)^T + (\mu^E - \mu^E_c)(f(x) - \mu^E)^T] \; \textnormal{d}x \\
=& (\mu^E - \mu^E_c)(\mu^E - \mu^E_c)^T + \int_{x} p_E(x|c) \cdot (f(x) - \mu^E)(f(x) - \mu^E)^T \; \textnormal{d}x + \int_{x} p_E(x|c) \cdot [(f(x) - \mu^E)(\mu^E - \mu^E_c)^T \\
&+ (\mu^E - \mu^E_c)(f(x) - \mu^E)^T] \; \textnormal{d}x \\
=& \int_{x} p_E(x|c) \cdot (f(x) - \mu^E)(f(x) - \mu^E)^T \; \textnormal{d}x  + (\mu^E - \mu^E_c)(\mu^E - \mu^E_c)^T + 2(\mu^E_c - \mu^E)(\mu^E - \mu^E_c)^T \\
=& \int_{x} p_E(x|c) \cdot (f(x) - \mu^E)(f(x) - \mu^E)^T \; \textnormal{d}x - (\mu^E - \mu^E_c)(\mu^E - \mu^E_c)^T
\end{aligned}
\end{equation}
Hence,
\begin{equation}
\begin{split}
\mathbb{E}_{c \sim \mathcal{D}_C}[\Sigma^E_{W,c}] 
=& \sum_c p(c) \cdot \int_{x} p_E(x|c) \cdot (f(x) - \mu^E)(f(x) - \mu^E)^T \;\textnormal{d}x - \sum_c p(c) \cdot (\mu^E_c - \mu^E)(\mu^E_c - \mu^E)^T = \Sigma^E - \Sigma^E_B
\end{split}
\end{equation}
In particular, 
\begin{equation}
\Sigma^E = \Sigma^E_B + \mathbb{E}_c[\Sigma^E_{W,c}] = \Sigma^E_B + \Sigma^E_W
\end{equation}
Therefore, we summarize:
\begin{equation}
\begin{split}
FID(\mathcal{D}_R,\mathcal{D}_G) =& \|\mu^R - \mu^G\|^2 + \Tr(\Sigma^R + \Sigma^G - 2(\Sigma^R\Sigma^G)^\frac{1}{2}) \\
    =& \|\mu^R_B - \mu^G_B\|^2 + \Tr\left(\Sigma^R_B + \mathbb{E}_c[\Sigma^R_{W,c}] \right) + \Tr\left(\Sigma^G_B + \mathbb{E}_c[\Sigma^G_{W,c}] \right) - 2\Tr\left( \left[\left(\mathbb{E}_c[\Sigma^R_{W,c}] + \Sigma^R_B \right) \left(\mathbb{E}_c[\Sigma^G_{W,c}] + \Sigma^G_B\right) \right]^\frac{1}{2} \right)
\end{split}
\end{equation}
Now we can say the following:
\begin{equation}
\begin{split}
FID(\mathcal{D}_R,\mathcal{D}_G) =& BCFID(\mathcal{D}_R,\mathcal{D}_G) + WCFID(\mathcal{D}_R,\mathcal{D}_G) - \sum_c p(c) \cdot \|\mu^R_c -\mu^G_c\|^2 \\
&- 2Tr\left( \left( \left(\mathbb{E}_c[\Sigma^R_{W,c}] + \Sigma^R_B \right) \cdot \left(\mathbb{E}_c[\Sigma^G_{W,c}] + \Sigma^G_B\right) \right)^\frac{1}{2} \right) + 2Tr\left(\sum_c{p(c)(\Sigma^R_{W,c}\Sigma^G_{W,c})^\frac{1}{2}} + (\Sigma^R_B\Sigma^G_B)^\frac{1}{2} \right)
\end{split}
\end{equation}
We denote:
\begin{equation}
\begin{split}
M   :=&  \sum_c p(c) \cdot \|\mu^R_c -\mu^G_c\|^2 \\
&+ 2\Tr\left( \left( \left(\mathbb{E}_c[\Sigma^R_{W,c}] + \Sigma^R_B \right) \cdot \left(\mathbb{E}_c[\Sigma^G_{W,c}] + \Sigma^G_B\right) \right)^\frac{1}{2} \right) \\
&- 2\Tr\left(\sum_c{p(c)(\Sigma^R_{W,c}\Sigma^G_{W,c})^\frac{1}{2}} + (\Sigma^R_B\Sigma^G_B)^\frac{1}{2} \right)
\end{split}
\end{equation}
Next, we would like to show that $M \ge 0$. We consider that $M$ sums the non-negative term $\sum_c{p(c)\|\mu^R_c -\mu^G_c\|^2}$ with the following term:
\begin{equation}
\begin{aligned}
&\Tr\left( \left( \left(\mathbb{E}_c[\Sigma^R_{W,c}] + \Sigma^R_B \right) \cdot \left(\mathbb{E}_c[\Sigma^G_{W,c}] + \Sigma^G_B\right) \right)^\frac{1}{2} \right) \\
&- \Tr\left(\sum_c{p(c)(\Sigma^R_{W,c}\Sigma^G_{W,c})^\frac{1}{2}} + (\Sigma^R_B\Sigma^G_B)^\frac{1}{2} \right)\\
=&\Tr\left( \left( \left(\sum_c p(c)\cdot \Sigma^R_{W,c} + \Sigma^R_B \right) \cdot \left(\sum_c p(c) \cdot \Sigma^G_{W,c} + \Sigma^G_B\right) \right)^\frac{1}{2} \right) \\
&- \Tr\left(\sum_c{(p(c) \cdot \Sigma^R_{W,c} \cdot p(c) \cdot \Sigma^G_{W,c})^\frac{1}{2}} + (\Sigma^R_B\Sigma^G_B)^\frac{1}{2} \right) \\
=& \Tr\left( \sum_c p(c)\left(\left(\Sigma^R_{W,c}\right)^\frac{1}{2} - \left(\Sigma^G_{W,c}\right)^\frac{1}{2} \right)^2 + \left(\left(\Sigma^R_{B}\right)^\frac{1}{2} - \left(\Sigma^G_{B}\right)^\frac{1}{2} \right)^2\right) \\
 &- \Tr\left( \left( (\sum_c p(c)\Sigma^R_{W,c} + \Sigma^R_B )^\frac{1}{2} - (\sum_c p(c)\Sigma^G_{W,c} + \Sigma^G_B)^\frac{1}{2} \right)^2 \right) \\
\end{aligned}
\end{equation}

Since the function $H(x_1, x_2) = \left( x_1^{1/2} - x_2^{1/2} \right)^2$ is convex, by Jensen's trace inequality, the above term is non-negative.
 
This implies the desired inequality:
\begin{equation}
FID(\mathcal{D}_R,\mathcal{D}_G) \leq BCFID(\mathcal{D}_R,\mathcal{D}_G) + WCFID(\mathcal{D}_R,\mathcal{D}_G)
\end{equation}

Finally, we would like to demonstrate the tightness of the bound, aside from the trivial case of all or some of the covariance matrices being $0$ and $\mu^R_c = \mu^G_c$ for all $c$. Consider a case where all of the matrices $\Sigma^E_B$ and $\Sigma^E_{W,c}$ are simultanously diagonalizable, i.e., there exist an invertible matrix $P$, such that:
\begin{equation}
\Sigma^E_B = P \cdot \Lambda^E_B \cdot  P^{-1} \textnormal{ and } \Sigma^E_{W,c} = P \cdot \Lambda^E_{W,c} \cdot P^{-1}
\end{equation}
where $\Lambda^E_B$ and $\Lambda^E_{W,c}$ are the diagonal matrices of the eigenvalues of $\Sigma^E_B$ and $\Sigma^E_{W,c}$ respectively.

Since all matrices are diagonal, we can rewrite $M$ as follows:
\begin{equation}
\begin{split}
M   :=&  \sum_c p(c) \cdot \|\mu^R_c -\mu^G_c\|^2 \\
&+ 2\Tr\left( \left( \left(\mathbb{E}_c[\Lambda^R_{W,c}] + \Lambda^R_B \right) \cdot \left(\mathbb{E}_c[\Lambda^G_{W,c}] + \Lambda^G_B\right) \right)^\frac{1}{2} \right) \\
&- 2\Tr\left(\sum_c{p(c)(\Lambda^R_{W,c}\Lambda^G_{W,c})^\frac{1}{2}} + (\Lambda^R_B\Lambda^G_B)^\frac{1}{2} \right) \\
=& \sum_c p(c) \cdot \|\mu^R_c -\mu^G_c\|^2 \\
&+ 2\sum^{k}_{d=1} \left( \left(\sum_c p(c) \cdot \sigma^R_{{W,c,d}} + \sigma^R_{B,d} \right)^{1/2} \left(\sum_c p(c)\cdot \sigma^G_{{W,c,d}} + \sigma^G_{B,d}\right)^{1/2}\right) \\
&- 2\sum^{k}_{d=1} \left(\sum_c p(c)\cdot (\sigma^R_{{W,c,d}})^{1/2} \cdot (\sigma^G_{{W,c,d}})^{1/2} + (\sigma^R_{B,d})^{1/2} (\sigma^G_{B,d})^{1/2} \right)
\end{split}
\end{equation}
where $\sigma^E_{m,d}$ denotes the $d$-th element on the diagonal of the matrix $\Lambda^E_{m}$ ($m$ is a specifier of the form $(W,c)$ or $B$). In addition, $k$ is the output dimension of $f$.

In addition, assume that (i) for each $d\in [k]$ there is only one member of $\{\sigma^R_{W,c,d}\}_{d=1}^{k} \cup \{\sigma^R_{B,d}\}$ that is nonzero, and the same for $\{\sigma^G_{W,c,d}\}_{d=1}^{k} \cup \{\sigma^G_{B,d}\}$ and, (ii) it has the same index. e.g. for $d=3$, only $\sigma^R_{{W,2,3}}$ is nonzero in the variances of the real data and the equivalent $\sigma^G_{{W,2,3}}$ is the only nonzero variance in the generated data. 

Finally, if we also have, $\mu^R_c = \mu^G_c$, for all $c$, then we get $M=0$.

For example, we consider the following setting. Let $\mathcal{D}_C$ be a distribution over two classes $c=1$ and $c=2$. Let $g$ be a function that satisfies:
$\mathbb{E}_{x \sim \mathcal{D}^c_E}[g(x)] = 1$ for all classes $c=1,2$ and specifiers $E = R,G$. We consider a function $f:\mathbb{R}^n\to\mathbb{R}^2$ that satisfies the following: for any sample $x \sim \mathcal{D}^1_E$ or $x \sim \mathcal{D}^2_E$, we have: $f(x) = (1,g(x))$ and $f(x) = (g(x),1)$ respectively. Hence, $\mu^E_c = (1,1)$ for all $c$ and $E=R,G$. Therefore, we have:
\begin{equation}
\sum_c p(c) \cdot \|\mu^R_c -\mu^G_c\|^2 = 0
\end{equation}
and also
\begin{equation}
\Sigma^{E}_{B} = 
\begin{pmatrix}
0 & 0 \\
0 & 0 \\
\end{pmatrix}, \;\;
\Sigma^{E}_{W,1} = 
\begin{pmatrix}
\sigma^E_1 & 0 \\
0 & 0 \\
\end{pmatrix}, \;\;
\Sigma^{E}_{W,2} = 
\begin{pmatrix}
0 & 0 \\
0 & \sigma^E_2 \\
\end{pmatrix}
\end{equation}
where $\sigma^E_i$ is the standard deviation of $g(x)$ for $x \sim \mathcal{D}^{c}_{E}$. Therefore, we have:
\begin{equation}
\begin{aligned}
&\Tr\left( \left( \left(\mathbb{E}_c[\Sigma^R_{W,c}] + \Sigma^R_B \right) \cdot \left(\mathbb{E}_c[\Sigma^G_{W,c}] + \Sigma^G_B\right) \right)^\frac{1}{2} \right) \\
=& \Tr \left( \left(\mathbb{E}_c[\Sigma^R_{W,c}] \cdot \mathbb{E}_c[\Sigma^G_{W,c}] \right)^\frac{1}{2} \right) \\
=& \sqrt{\sigma^R_1 \cdot \sigma^R_2} + \sqrt{\sigma^{G}_1 \cdot \sigma^G_2} \\
=& \Tr\left(\sum_c{p(c)(\Sigma^R_{W,c}\Sigma^G_{W,c})^\frac{1}{2}} \right)\\
=& \Tr\left(\sum_c{p(c)(\Sigma^R_{W,c}\Sigma^G_{W,c})^\frac{1}{2}} + (\Sigma^R_B\Sigma^G_B)^\frac{1}{2} \right) 
\end{aligned}
\end{equation}
We conclude that $M=0$ as desired.
\end{proof}

Naturally, these cases assume perfect alignment between the real and generated data and also between and within classes in each dataset, which is very unlikely. This really highlights how much information is lost when measuring FID instead of the BCFID + WCFID as an evaluation score.

\end{document}